\documentclass[12pt,a4paper]{article}
\usepackage{amsmath}
\usepackage{times}
\usepackage{graphicx}
\usepackage{color}
\usepackage{multirow}
\usepackage[authoryear]{natbib}
\usepackage{rotating}
\usepackage{bbm}

\newcommand{\captionfonts}{\normalsize}

\makeatletter  
\long\def\@makecaption#1#2{%
  \vskip\abovecaptionskip
  \sbox\@tempboxa{{\captionfonts #1: #2}}%
  \ifdim \wd\@tempboxa >\hsize
    {\captionfonts #1: #2\par}
  \else
    \hbox to\hsize{\hfil\box\@tempboxa\hfil}%
  \fi
  \vskip\belowcaptionskip}
\makeatother   

\usepackage{mathrsfs,amsmath,amsfonts,amssymb}
\usepackage{MnSymbol,mathtools,bm,xcolor}
\usepackage{breakcites}
\usepackage{breakcites}
\usepackage{enumerate} 
\usepackage{tabularx,multirow,array,booktabs}
\usepackage{url}
\usepackage{breakurl}
\DeclareUrlCommand\email{}
\usepackage{colortbl}
\usepackage{graphicx}
\graphicspath{{figures/}}%
\usepackage{float}
\usepackage{subcaption} 
\usepackage{amsthm}
\theoremstyle{plain}
\newtheorem{theorem}{Theorem}[section]%
\newtheorem{corollary}[theorem]{Corollary} %
\newtheorem{lemma}[theorem]{Lemma}

\newtheorem{proposition}[theorem]{Proposition}
\theoremstyle{definition}

\theoremstyle{remark}

\newcommand{\R}{\ifmmode\mathbb{R}\else$\mathbb{R}$\fi}
\newcommand{\N}{\ifmmode\mathbb{N}\else$\mathbb{N}$\fi}
\newcommand{\Z}{\ifmmode\mathbb{Z}\else$\mathbb{Z}$\fi}
\newcommand{\Q}{\ifmmode\mathbb{Q}\else$\mathbb{Q}$\fi}
\let\tn\textnormal
\newcommand{\bmx}{{\bm{x}}}

\newcommand{\bmy}{{\bm{y}}}

\newcommand{\bmW}{{\bm{W}}}

\newcommand{\bmzero}{{\bm{0}}}

\newcommand{\bmbeta}{{\bm{\beta}}}

\newcommand{\bmPhi}{{\bm{\Phi}}}

\newcommand{\calL}{{\mathcal{L}}}
\newcommand{\calO}{{\mathcal{O}}}

\newcommand{\tildephi}{{\widetilde{\phi}}}

\newcommand{\tildeL}{{\widetilde{L}}}
\newcommand{\tildef}{{\widetilde{f}}}

\newcommand{\tildeN}{{\widetilde{N}}}

\newcommand{\barN}{{\bar{N}}}
\newcommand{\barL}{{\bar{L}}}

\DeclareMathOperator*{\argmin}{arg\,min}
\newcommand{\bin}{\tn{bin}\hspace{1.2pt}}
\newcommand{\sj}[1]{{\color{black}#1}}

\definecolor{mygreen}{HTML}{00DD00}

\newcommand{\mystep}[2]{\par \vspace{0.25cm}\noindent\textbf{\hspace{8pt}Step }$#1\colon$ #2 \vspace{0.18cm} \par }

\newcommand{\myto}[2][1]{\mathop{\raisebox{-0pt}{\scalebox{#2}[#1]{$\longrightarrow$}}}}
\usepackage[pdftex,bookmarks,colorlinks,breaklinks=true]{hyperref}
\definecolor{myturquoise}{HTML}{008080}
\hypersetup{citecolor=myturquoise,linkcolor=blue,urlcolor=red}
\usepackage[left,mathlines]{lineno}
\definecolor{mycolor}{HTML}{BEBEBE}

\makeatletter
\newcommand*\patchAmsMathEnvironmentForLineno[1]{%
	\expandafter\let\csname old#1\expandafter\endcsname\csname #1\endcsname
	\expandafter\let\csname oldend#1\expandafter\endcsname\csname end#1\endcsname
	\renewenvironment{#1}%
	{\linenomath\csname old#1\endcsname}%
	{\csname oldend#1\endcsname\endlinenomath}}%
\newcommand*\patchBothAmsMathEnvironmentsForLineno[1]{%
	\patchAmsMathEnvironmentForLineno{#1}%
	\patchAmsMathEnvironmentForLineno{#1*}}%
\patchBothAmsMathEnvironmentsForLineno{equation}%
\patchBothAmsMathEnvironmentsForLineno{align}%
\patchBothAmsMathEnvironmentsForLineno{flalign}%
\patchBothAmsMathEnvironmentsForLineno{alignat}%
\patchBothAmsMathEnvironmentsForLineno{gather}%
\patchBothAmsMathEnvironmentsForLineno{multline}%
\makeatother

\let\epsilon\varepsilon

\let\cite\citep
\let\subset\subseteq

\begin{document}    
\hspace{13.9cm}1

\ \vspace{20mm}\\

{\LARGE Deep Network with Approximation Error Being Reciprocal of Width to Power of  Square Root of Depth}

\ \\
{\bf \large Zuowei Shen}\\
{matzuows@nus.edu.sg}\\
{Department of Mathematics,  National University of Singapore}\\
\ \\
{\bf \large Haizhao Yang}\\
{haizhao@purdue.edu}\\
{Department of Mathematics, Purdue University}\\
\ \\
{\bf \large Shijun Zhang}\\
{zhangshijun@u.nus.edu}\\
{Department of Mathematics,  National University of Singapore}\\


\thispagestyle{empty}
\markboth{}{NC instructions}
\ \vspace{-0mm}\\
%
\begin{center} {\bf Abstract} \end{center}
 A new network with super approximation power is introduced. This network is built with Floor ($\lfloor x\rfloor$) or ReLU ($\max\{0,x\}$) activation function in each neuron and hence we call such networks Floor-ReLU networks. For any hyper-parameters $N\in\mathbb{N}^+$ and $L\in\mathbb{N}^+$, it is shown that Floor-ReLU networks with width $\max\{d,\, 5N+13\}$ and depth $64dL+3$ can uniformly approximate a H\"older function $f$ on $[0,1]^d$ with {an approximation error} $3\lambda d^{\alpha/2}N^{-\alpha\sqrt{L}}$, where $\alpha \in(0,1]$ and $\lambda$ are the H\"older order and constant, respectively. More generally  for an arbitrary continuous function $f$ on $[0,1]^d$ with a modulus of continuity $\omega_f(\cdot)$, the constructive approximation rate is $\omega_f(\sqrt{d}\,N^{-\sqrt{L}})+2\omega_f(\sqrt{d}){N^{-\sqrt{L}}}$. As a consequence, this new class of networks overcomes the curse of dimensionality in approximation power when the variation of $\omega_f(r)$ as $r\to 0$ is moderate (e.g., $\omega_f(r) \lesssim r^\alpha$ for H\"older continuous functions), since the major term to be {considered} in our approximation rate is essentially $\sqrt{d}$ times a function of $N$ and $L$ independent of $d$ within the modulus of continuity. 

\section{Introduction}
\label{sec:introduction}

Recently, there has been a large number of successful real-world applications of deep neural networks in many fields of computer science and engineering, especially for large-scale and high-dimensional learning problems.  Understanding the approximation capacity of deep neural networks has become a fundamental research direction for revealing the advantages of deep learning {compared to} traditional methods. This paper introduces new theories and network architectures achieving root exponential convergence and avoiding the curse of dimensionality \sj{simultaneously for (H\"older) continuous functions with an explicit error bound} in deep network approximation, which might be two foundational laws supporting the application of deep network approximation in large-scale and high-dimensional problems. The approximation {results} here are quantitative and {apply} to networks with essentially arbitrary width and depth. {These results suggest considering Floor-ReLU networks as a
possible alternative to ReLU networks in deep learning. }

Deep ReLU networks with  width $\calO(N)$ and  depth $\calO(L)$ can achieve the approximation rate $\calO(N^{-L})$ for polynomials on $[0,1]^d$ \cite{Shen3}  but it is not true for general functions, e.g., the (nearly) optimal approximation rates of deep ReLU networks for a \sj{Lipschitz} continuous function and a $C^s$ function $f$ on $[0,1]^d$ are $\calO(\sqrt{d}N^{-2/d}L^{-2/d})$ and $\calO(\|f\|_{C^s}N^{-2s/d}L^{-2s/d})$  \cite{Shen2,Shen3}, respectively. The limitation of ReLU networks motivates us to explore other types of network architectures to answer \sj{our curiosity on deep networks: Do deep neural networks with arbitrary width $\calO(N)$ and arbitrary depth $\calO(L)$ admit an exponential approximation rate $\calO(\omega_f(N^{-L^\eta}))$ for some constant $\eta>0$ for a generic continuous function $f$ on $[0,1]^d$ with a modulus of continuity $\omega_f(\cdot)$?}

{To answer this question}, we introduce the Floor-ReLU network, which is a fully connected neural network (FNN) built with either Floor ($\lfloor x\rfloor$) or ReLU ($\max\{0,x\}$) activation function\footnote{Our results can be easily generalized to Ceiling-ReLU networks, namely, feed-forward neural networks with either Ceiling $(\lceil x \rceil$) or ReLU ($\max\{0,x\}$) activation function in each neuron.} in each neuron. Mathematically, if we let $N_0=d$, $N_{L+1}=1$, and  $N_\ell$ be the number of  neurons in $\ell$-th hidden layer of a Floor-ReLU network for $\ell=1,2,\cdots,L$, then the architecture of this network with input $\bmx$ and output $\phi(\bmx)$ can be described as 
\begin{equation*}
    \begin{aligned}
    \bmx=\widetilde{\bm{h}}_0 
    \myto{1.85}^{\bmW_0,\, \bm{b}_0} \bm{h}_1
    \myto{1.8765}^{\tn{$\sigma$ or  $\lfloor\cdot\rfloor$}} \widetilde{\bm{h}}_1 \quad 
    \cdots \quad
    \myto{2.54}^{\bmW_{L-1},\, \bm{b}_{L-1}} \bm{h}_L
    \myto{1.8765}^{\tn{$\sigma$ or $\lfloor\cdot\rfloor$}} 
    \widetilde{\bm{h}}_L 
    \myto{1.852}^{\bmW_{L},\, \bm{b}_{L}} \bm{h}_{L+1}=\phi(\bmx),
    \end{aligned}
\end{equation*}
    where $\bmW_\ell\in \R^{N_{\ell+1}\times N_{\ell}}$, $\bm{b}_\ell\in \R^{N_{\ell+1}}$, $\bm{h}_{\ell+1} :=\bmW_\ell\cdot \widetilde{\bm{h}}_{\ell} + \bm{b}_\ell$ for 
    $\ell=0,1,\cdots,L$, and $\widetilde{\bm{h}}_{\ell,n}$ is equal to $\sigma(\bm{h}_{\ell,n})\tn{ or } \lfloor \bm{h}_{\ell,n}\rfloor$ for $\ell=1,2,\cdots,L$ and $n=1,2,\cdots,N_\ell$, where  $\bm{h}_\ell=(\bm{h}_{\ell,1},\cdots,\bm{h}_{\ell,N_\ell})$ and $\widetilde{\bm{h}}_\ell=(\widetilde{\bm{h}}_{\ell,1},\cdots,\widetilde{\bm{h}}_{\ell,N_\ell})$ for $\ell=1,2,\cdots,L$. 
	See Figure \ref{fig:floorReLUeg} for an example.
	
	\begin{figure}[!htp]
		\centering            
		\includegraphics[width=0.8\textwidth]{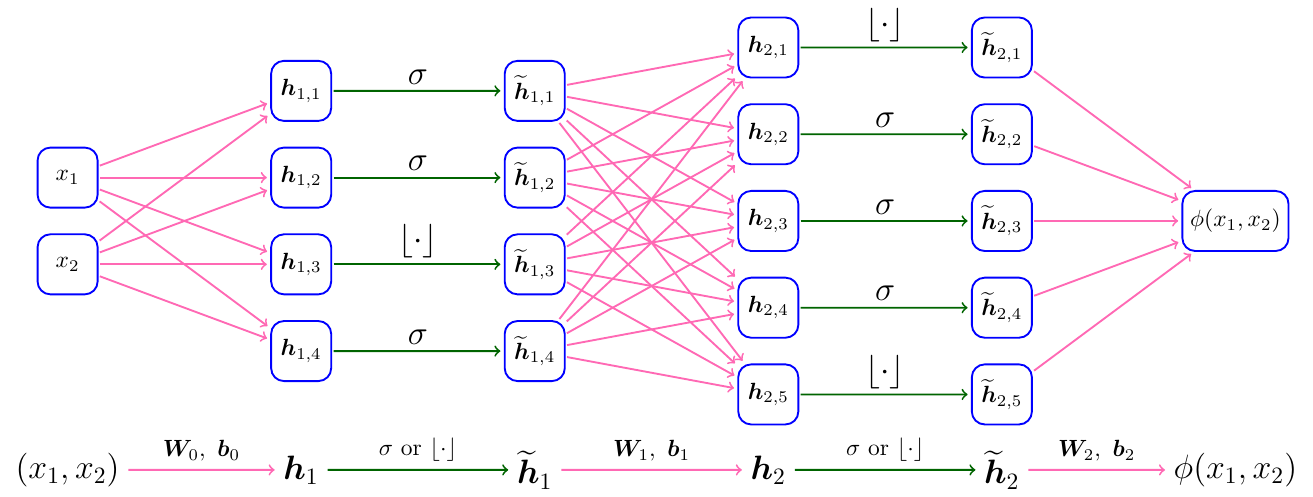}
		\caption{An example of {a} Floor-ReLU network with  width $5$ and  depth $2$.
		}
		\label{fig:floorReLUeg}
	\end{figure}

In Theorem \ref{thm:main} below, we show by construction that Floor-ReLU networks with  width $\max\{d,\, 5N+13\}$ and  depth $64dL+3$ can \sj{uniformly} approximate a continuous function $f$ on $[0,1]^d$ with a \sj{root} exponential approximation rate\footnote{All the exponential convergence in this paper is root exponential convergence. {Nevertheless, after the introduction,
for the convenience of presentation, we will omit the prefix ``root'', as in the literature.}} $\omega_f(\sqrt{d}\,N^{-\sqrt{L}})+2\omega_f(\sqrt{d}){N^{-\sqrt{L}}}$, where $\omega_f(\cdot)$ is the modulus of continuity defined as 
\begin{equation*}
\omega_f(r)\coloneqq \sup\big\{|f(\bmx)-f(\bmy)|: \|\bmx-\bmy\|_2\le r,\ \bmx,\bmy\in [0,1]^d\big\},\quad\tn{for any $r\ge0$,}
\end{equation*}
where $\|\bmx\|_2=\sqrt{x_1^2+x_2^2+\cdots+x_d^2}$ for any $\bmx=(x_1,x_2,\cdots,x_d)\in \R^d$. 
\begin{theorem}
    \label{thm:main}
    Given any $N,L\in \N^+$ and an arbitrary continuous function $f$ on $[0,1]^d$, there exists a function $\phi$ implemented by a Floor-ReLU network  with  width $\max\{d,\,5N+13\}$ and  depth $64dL+3$ such that
    \begin{equation*}
    |\phi(\bmx)-f(\bmx)|\le \omega_f(\sqrt{d}\,N^{-\sqrt{L}})+2\omega_f(\sqrt{d}){N^{-\sqrt{L}}},\quad \tn{for any $\bmx\in [0,1]^d$.}
    \end{equation*}
\end{theorem}

With Theorem \ref{thm:main}, we have an immediate corollary.
\begin{corollary}
    \label{coro:main}
    Given an arbitrary continuous function $f$ on $[0,1]^d$, there exists a function $\phi$ implemented by a Floor-ReLU network  with  width $\barN$ and  depth $\barL$ such that
    \begin{equation*}
    {|\phi(\bmx)-f(\bmx)|\le \omega_f\Big(\sqrt{d}\,{\big\lfloor\tfrac{\barN-13}{5}\big\rfloor}^{-\sqrt{\big\lfloor\tfrac{\barL-3}{64d}\big\rfloor}}\Big)+2\omega_f(\sqrt{d}){{\big\lfloor\tfrac{\barN-13}{5}\big\rfloor}^{-\sqrt{\big\lfloor\tfrac{\barL-3}{64d}\big\rfloor}}},}
    \end{equation*}
    for any $\bmx\in [0,1]^d$ and { $\barN,\barL\in \N^+$ with $\barN\ge \max\{d,18\}$ and $\barL\ge 64d+3$.}
\end{corollary}

In Theorem \ref{thm:main}, the rate in $\omega_f(\sqrt{d}\,N^{-\sqrt{L}})$ implicitly depends on $N$ and $L$ through the modulus of continuity of $f$, while the rate in $2\omega_f(\sqrt{d}){N^{-\sqrt{L}}}$ is \sj{explicit in $N$ and $L$.} Simplifying the implicit approximation rate to make it explicitly depending on $N$ and $L$ is challenging in general. \sj{However, if $f$ is a H{\"o}lder continuous function on $[0,1]^d$ of order $\alpha\in(0,1]$ with a constant $\lambda$, i.e., $f(\bmx)$ satisfying
\begin{equation}\label{eqn:Holder}
|f(\bmx)-f(\bmy)|\leq \lambda \|\bmx-\bmy\|_2^\alpha,\quad \tn{for any $\bmx,\bmy\in[0,1]^d$,}
\end{equation}
then $\omega_f(r)\le \lambda r^\alpha$ for any $r\ge 0$. Therefore, in the case of H{\"o}lder continuous functions, the approximation rate is simplified to $3\lambda d^{\alpha/2}N^{-\alpha\sqrt{L}}$ as shown in the following corollary. In the special case of Lipschitz continuous functions with a Lipschitz constant $\lambda$, the approximation rate is simplified to $3\lambda\sqrt{d}N^{-\sqrt{L}}$.}

\begin{corollary}
    \label{coro:main}
    Given any $N,L\in \N^+$ and a H{\"o}lder continuous function $f$ on $[0,1]^d$ of order $\alpha$ with a constant $\lambda$, there exists a function $\phi$ implemented by a Floor-ReLU network  with  width $\max\{d,\,5N+13\}$ and  depth $64dL+3$ such that
    \begin{equation*}
    |\phi(\bmx)-f(\bmx)|\le 3\lambda d^{\alpha/2}{N^{-\alpha\sqrt{L}}},\quad \tn{for any $\bmx\in [0,1]^d$.}
    \end{equation*}
\end{corollary}

First, Theorem \ref{thm:main} and Corollary \ref{coro:main} show that the approximation capacity of deep networks for continuous functions can be nearly exponentially improved by increasing the network depth, and the approximation error can be explicitly characterized in terms of the width $\calO(N)$ and depth $\calO(L)$. Second, this new {class of} network{s} overcomes the curse of dimensionality in the approximation power when the modulus of continuity is moderate, since the approximation order is {essentially $\omega_f(\sqrt{d} N^{-\sqrt{L}})$.}  Finally, applying piecewise constant and integer-valued functions as activation functions and integer numbers as parameters {has} been explored in {the study of} quantized neural networks \cite{qnn,Yin2019UnderstandingSE,2013arXiv1308.3432B} with efficient training algorithms for low computational complexity \cite{qnn2}. {The floor function ($\lfloor x\rfloor$)} is a piecewise constant function and can be easily implemented numerically at very little cost. Hence, the evaluation of the proposed network could be efficiently implemented in practical computation. Though there might not be an existing optimization algorithm to identify an approximant with the approximation rate in this paper, Theorem \ref{thm:main} can provide an expected accuracy before a learning task and how much the current optimization algorithms could be improved. \sj{Designing an efficient optimization algorithm for Floor-ReLU networks will be left as future work with several possible directions discussed later.}

{We would like to remark that an increased smoothness or
regularity of the target function could improve our approximation rate but at the cost of a
large prefactor.}
For example, to attain better approximation rates for functions in $C^s([0,1]^d)$, it is common to use Taylor expansions and derivatives, which are tools that suffer from the curse of dimensionality and will result in a large prefactor like $\calO((s+1)^d)$ {that is subject to the curse of dimensionality}. Furthermore, the prospective approximation rate using smoothness is not attractive. For example, the prospective approximation rate would be $\calO(N^{-s\sqrt{L}})$, if we use Floor-ReLU networks with  width $\calO(N)$ and  depth $\calO(L)$ to approximate functions in $C^s([0,1]^d)$. However, such a rate $\calO(N^{-s\sqrt{L}})=\calO(N^{-\sqrt{s^2L}})$ can be attained by using Floor-ReLU networks with  width $\calO(N)$ and  depth $\calO(s^2L)$ to approximate Lipschitz continuous functions. Hence, increasing the network depth can result in the same approximation rate for Lipschitz continuous functions as the rate of smooth functions.

The rest of this paper is organized as follows. In Section \ref{sec:dis}, we discuss the application scope of our theory and compare related works in the literature. In Section \ref{sec:approxContFunc}, we prove Theorem \ref{thm:main}  based on  Proposition \ref{prop:bitsExtraction}. Next, this basic proposition is proved in Section \ref{sec:proofOfBitsExtraction}. 
Finally, we conclude this paper in Section~\ref{sec:conclusion}.

\section{Discussion}
\label{sec:dis}
\sj{In this section, we will discuss the application scope of our theory in machine learning and its comparison related to existing works.}

\subsection{Application scope of our theory in machine learning}

In supervised learning, an unknown target function $f(\bm{x})$ defined on a domain $\Omega$ is learned through its finitely many samples $\{( \bm{x}_i,f(\bm{x}_i){ )}\}_{i=1}^n$. If deep networks are applied in supervised learning, the following optimization problem is solved to identify a deep network $\phi(\bm{x};\bm{\theta}_{\mathcal{S}})${,} with $\bm{\theta}_{\mathcal{S}}$ as the set of parameters{,} to infer $f(\bm{x})$ for unseen data samples $\bm{x}$:
\begin{equation}\label{eqn:emloss}
\bm{\theta}_{\mathcal{S}}=\argmin_{\bm{\theta}}R_{\mathcal{S}}(\bm{\theta}):=\argmin_{\bm{\theta}}
    \frac{1}{n}\sum_{\{\bm{x}_i\}_{i=1}^n} \ell\big( \phi(\bm{x}_i;\bm{\theta}),f(\bm{x}_i)\big) 
\end{equation}
with a loss function typically taken as $\ell(y,y')=\frac{1}{2}|y-y'|^2$. The inference error is usually measured by $R_{\mathcal{D}}(\bm{\theta}_{\mathcal{S}})$, where
\begin{equation*}
    R_{\mathcal{D}}(\bm{\theta}):=
    \tn{E}_{\bm{x}\sim U(\Omega)} \left[\ell( \phi(\bm{x};\bm{\theta}),f(\bm{x}))\right],
\end{equation*}
where the expectation is taken with an {unknown data distribution $U(\Omega)$} over $\Omega$. 

Note that the best deep network to infer $f(\bm{x})$ is $\phi(\bm{x};\bm{\theta}_{\mathcal{D}})$ with $\bm{\theta}_{\mathcal{D}}$ given by
\begin{equation*}
    \bm{\theta}_{\mathcal{D}}=\argmin_{\bm{\theta}} R_{\mathcal{D}}(\bm{\theta}).
\end{equation*}
The best possible inference error is $R_{\mathcal{D}}(\bm{\theta}_{\mathcal{D}})$. In real applications, $U(\Omega)$ is unknown and only finitely many samples from this distribution are available. Hence, the empirical loss $R_{\mathcal{S}}(\bm{\theta})$ is minimized hoping to obtain $\phi(\bm{x};\bm{\theta}_{\mathcal{S}})$, instead of minimizing the population loss $R_{\mathcal{D}}(\bm{\theta})$ to obtain $\phi(\bm{x};\bm{\theta}_{\mathcal{D}})$. In practice, a numerical optimization method to solve \eqref{eqn:emloss} may result in a numerical solution (denoted as $\bm{\theta}_{\mathcal{N}}$) that may not be a global minimizer $\bm{\theta}_{\mathcal{S}}$. Therefore, the actually learned neural network to infer $f(\bm{x})$ is $\phi(\bm{x};\bm{\theta}_{\mathcal{N}})$ and the corresponding inference error is measured by $R_{\mathcal{D}}(\bm{\theta}_{\mathcal{N}})$. 
 
By the discussion just above, it is crucial to quantify $R_{\mathcal{D}}(\bm{\theta}_{\mathcal{N}})$ to see how good the learned neural network $\phi(\bm{x};\bm{\theta}_{\mathcal{N}})$ is, since $R_{\mathcal{D}}(\bm{\theta}_{\mathcal{N}})$ is the expected inference error over all possible data samples. Note that
\begin{align}\label{eqn:gen}
        \quad
        R_{\mathcal{D}}(\bm{\theta}_{\mathcal{N}})
        &=[R_{\mathcal{D}}(\bm{\theta}_{\mathcal{N}})-R_{\mathcal{S}}(\bm{\theta}_{\mathcal{N}})]
        +[R_{\mathcal{S}}(\bm{\theta}_{\mathcal{N}})-R_{\mathcal{S}}(\bm{\theta}_{\mathcal{S}})]   +[R_{\mathcal{S}}(\bm{\theta}_{\mathcal{S}})-R_{\mathcal{S}}(\bm{\theta}_{\mathcal{D}})]    \nonumber\\
   &\quad  
        +[R_{\mathcal{S}}(\bm{\theta}_{\mathcal{D}})-R_{\mathcal{D}}(\bm{\theta}_{\mathcal{D}})]+R_{\mathcal{D}}(\bm{\theta}_{\mathcal{D}}) \nonumber \\
        \begin{split}
        &\leq 
        R_{\mathcal{D}}(\bm{\theta}_{\mathcal{D}})   +[R_{\mathcal{S}}(\bm{\theta}_{\mathcal{N}})-R_{\mathcal{S}}(\bm{\theta}_{\mathcal{S}})] \\
        &\quad+[R_{\mathcal{D}}(\bm{\theta}_{\mathcal{N}})-R_{\mathcal{S}}(\bm{\theta}_{\mathcal{N}})]
        +[R_{\mathcal{S}}(\bm{\theta}_{\mathcal{D}})-R_{\mathcal{D}}(\bm{\theta}_{\mathcal{D}})],
        \end{split}
\end{align}
where the inequality comes from the fact that $[R_{\mathcal{S}}(\bm{\theta}_{\mathcal{S}})-R_{\mathcal{S}}(\bm{\theta}_{\mathcal{D}})]\leq 0$ since $\bm{\theta}_{\mathcal{S}}$ is a global minimizer of $R_{\mathcal{S}}(\bm{\theta})$. The constructive approximation established in this paper and in the literature provides an upper bound of $R_{\mathcal{D}}(\bm{\theta}_{\mathcal{D}})$ in terms of the network size, e.g., in terms of the network width and depth, or in terms of the number of parameters. The second term of \eqref{eqn:gen} is bounded by the optimization error of the numerical algorithm applied to solve the empirical loss minimization problem in \eqref{eqn:emloss}. If the numerical algorithm is able to find a global minimizer, the second term is equal to zero. The theoretical guarantee of the convergence of an optimization algorithm to a global minimizer $\bm{\theta}_{\mathcal{S}}$ and the characterization of the convergence belong to the optimization analysis of neural networks. The third and fourth term of \eqref{eqn:gen} are usually bounded in terms of the sample size $n$ and a certain norm of $\bm{\theta}_{\mathcal{N}}$ and $\bm{\theta}_{\mathcal{D}}$ {(e.g., $\ell_1$, $\ell_2$, or the path norm)}, respectively. The study of the bounds for the third and fourth terms is referred to as the generalization error analysis of neural networks. 

The approximation theory, the optimization theory, and the generalization theory form the three main theoretical aspects of deep learning with different emphases and challenges, which have motivated many separate research directions recently. Theorem \ref{thm:main} and Corollary \ref{coro:main} provide an upper bound of $R_{\mathcal{D}}(\bm{\theta}_{\mathcal{D}})$. This bound only depends on the given budget of neurons {and layers} of Floor-ReLU networks {and on the modulus of continuity of the target function $f$}. Hence, this bound is independent of the empirical loss minimization in \eqref{eqn:emloss} and the optimization algorithm used to compute the numerical solution of \eqref{eqn:emloss}. {In other words, Theorem \ref{thm:main} and Corollary \ref{coro:main} quantify the approximation power of Floor-ReLU networks with a given size.} Designing efficient optimization algorithms and analyzing the generalization bounds for Floor-ReLU networks are two other separate future directions. Although optimization algorithms and generalization analysis are not our focus in this paper, {in the
next two paragraphs, we discuss several possible research topics in these directions for our Floor-ReLU networks.}

In this work, we have not analyzed the feasibility of optimization algorithms for the Floor-ReLU network. Typically, stochastic gradient descent (SGD) is applied to solve a network optimization problem. However, the Floor-ReLU network has piecewise constant activation functions making standard SGD infeasible. There are two possible directions to solve the optimization problem for Floor-ReLU networks: 1) gradient-free optimization methods, e.g., Nelder-Mead method \cite{Nelder:1965zz}, genetic algorithm \cite{10.2307/24939139}, simulated annealing \cite{Kirkpatrick671}, particle swarm optimization \cite{488968}, and consensus-based optimization \cite{doi:10.1142/S0218202517400061,Carrillo2019ACG}; 2) applying optimization algorithms for quantized networks that also have piecewise constant activation functions \cite{8955646,Boo2020QuantizedNN,2013arXiv1308.3432B,qnn2,qnn,Yin2019UnderstandingSE}. It would be interesting future work to explore efficient learning algorithms based on the Floor-ReLU network.

Generalization analysis of Floor-ReLU networks is also an interesting future direction. 
Previous works  have shown the generalization power of ReLU networks for regression problems \cite{Arthur18,Yuan1,Chen1,Weinan2019,e2020representation} and for solving partial differential equations \cite{DBLP:journals/corr/abs-1809-03062,Luo2020TwoLayerNN}. Regularization strategies for ReLU networks to guarantee good generalization capacity of deep learning have been proposed in \cite{Weinan2019,e2020representation}. It is important to investigate the generalization capacity of our Floor-ReLU networks. Especially, it is of great interest to see whether problem-dependent regularization strategies exist to make the generalization error of our Floor-ReLU networks free of the curse of dimensionality. 

\subsection{Approximation rates in $\calO(N)$ and $\calO(L)$ versus $\calO(W)$}

Characterizing deep network approximation in terms of the width $\calO(N)$\footnote{For simplicity, we omit $\calO(\cdot)$ in the following discussion.} and depth $\calO(L)$ simultaneously is fundamental and indispensable in realistic applications, while quantifying the deep network approximation based on the number of nonzero parameters $W$ is probably only of interest in theory as far as we know. Theorem \ref{thm:main} can provide practical guidance for choosing network sizes in realistic applications while theories in terms of $W$ cannot tell how large a network should be to guarantee a target accuracy. The width and depth are the two most direct and amenable hyper-parameters in choosing a specific network for a learning task, while the number of nonzero parameters $W$ is hardly controlled efficiently. Theories in terms of $W$ essentially have a single variable to control the network size in three types of structures: 1) fixing the width $N$ and varying the depth $L$; 2) fixing the depth $L$ and changing the width $N$; 3) both the width and depth are controlled by the same parameter like the target accuracy $\epsilon$ in a specific way (e.g., $N$ is a polynomial of $\frac{1}{\epsilon^d}$ and $L$ is a polynomial of $\log(\frac{1}{\epsilon})$). Considering the non-uniqueness of structures for realizing the same $W$, it is impractical to develop approximation rates in terms of $W$ covering all these structures. If one network structure has been chosen in a certain application, there might not be a known theory in terms of $W$ to quantify the performance of this structure. \sj{Finally, in terms of full error analysis of deep learning including approximation theory, optimization theory, and generalization theory as illustrated in \eqref{eqn:gen}, the approximation error characterization in terms of width and depth is more useful than that in terms of the number of parameters, because {almost all} existing optimization and generalization analysis are based on depth and width instead of the number of parameters  \cite{Arthur18,Yuan1,Chen1,Arora2019FineGrainedAO,AllenZhu2019LearningAG,Weinan2019,e2020representation,Ji2020PolylogarithmicWS}, to the best of our knowledge. Approximation results in terms of width and depth are more consistent with optimization and generalization analysis tools to obtain a full error analysis in \eqref{eqn:gen}.}

Most existing approximation theories for deep neural networks so far focus on the approximation rate in the number of parameters $W$ \cite{Cybenko1989ApproximationBS,HORNIK1989359,barron1993,DBLP:journals/corr/LiangS16,yarotsky2017,poggio2017,DBLP:journals/corr/abs-1807-00297,PETERSEN2018296,10.3389/fams.2018.00014,yarotsky18a,Ryumei,2019arXiv190501208G,2019arXiv190207896G,Wenjing,2019arXiv191210382L,suzuki2018adaptivity,Bao2019ApproximationAO,Opschoor2019,yarotsky2019,doi:10.1137/18M118709X,Hadrien,doi:10.1002/mma.5575,ZHOU2019,MO,bandlimit}. From the point of view of theoretical difficulty, controlling two variables $N$ and $L$ in our theory is more challenging than controlling one variable $W$ in the literature. In terms of mathematical logic, the characterization of deep network approximation in terms of $N$ and $L$ \sj{can provide an approximation rate} in terms of $W$, while \sj{we are not aware of how to derive approximation rates in terms of arbitrary $N$ and $L$ given approximation rates in terms of $W$, since existing results in terms of $W$ are valid for specific network sizes with  width and  depth as functions in $W$ without the degree of freedom to take arbitrary values.} As we have discussed in the last paragraph, existing theories essentially have a single variable to control the network size in three types of structures. \sj{Let us use the first type of structures, which includes the best-known result for a nearly optimal approximation rate, $\calO(\omega_f(W^{-2/d}))$, for continuous functions in terms of $W$  using ReLU networks \cite{yarotsky18a} and the best-known result, $\calO(\text{exp}(-c_{\alpha,d}\sqrt{W}))$, for H{\"o}lder continuous functions of order $\alpha$ using Sine-ReLU networks \cite{yarotsky2019}, as an example to show how Theorem \ref{thm:main} in terms of $N$ and $L$ can be applied to show a  better result in terms of $W$.} {One can apply Theorem \ref{thm:main} in a similar way} to obtain other corollaries with other types of structures in terms of $W$. The main idea is to specify the value of $N$ and $L$ in Theorem \ref{thm:main} to show the desired corollary. For example, {if} we let the width parameter $N=2$ and the depth parameter $L=W$ in Theorem \ref{thm:main}, then the width is $\max\{d, 23\}$, the depth is $64dW+3$, and the total number of parameters is bounded by $\calO\left(\max\{d^2, 23^2\}(64dW+3)\right)=\calO(W)$. \sj{Therefore, we can prove Corollary \ref{coro:W} below for the approximation capacity of our Floor-ReLU networks in terms of the total number of parameters as follows. }

\begin{corollary}
    \label{coro:W}
    Given any $W\in \N^+$ and a continuous function $f$ on $[0,1]^d$, there exists a function $\phi$ implemented by a Floor-ReLU network  with $\calO(W)$ nonzero parameters, a width $\max\{d,\,23\}$ and  depth $64dW+3$, such that
    \begin{equation*}
    |\phi(\bmx)-f(\bmx)|\le \omega_f(\sqrt{d}\,2^{-\sqrt{W}})+2\omega_f(\sqrt{d}){2^{-\sqrt{W}}},\quad \tn{for any $\bmx\in [0,1]^d$.}
    \end{equation*}
\end{corollary}

\sj{Corollary \ref{coro:W} achieves root exponential convergence without the curse of dimensionality in terms of the number of parameters $W$ {with the help of the Floor-ReLU networks. When only ReLU networks are used,} the result in \cite{yarotsky18a} suffers from the curse and does not have any kind of exponential convergence. The result in \cite{yarotsky2019} {with Sine-ReLU networks} has root exponential convergence but has not excluded the possibility of the curse of dimensionality as we shall discuss later. Furthermore, Corollary \ref{coro:W} works for generic continuous functions while \cite{yarotsky2019} {only applies to} H{\"o}lder continuous functions. }

\subsection{Further interpretation of our theory}

\sj{In the interpretation of our theory, there are two more aspects that are important to discuss. The first one is whether it is possible to extend our theory to functions on a more general domain, e.g, $[-M,M]^d$ for some $M>1$, because $M>1$ may cause an implicit curse of dimensionality in some existing theory as we shall point out later. The second one is how bad the modulus of continuity would be since it is related to a high-dimensional function $f$ that may lead to an implicit curse of dimensionality in our approximation rate. }

\sj{First, Theorem \ref{thm:main} can be easily generalized to $C([-M,M]^d)$ for any $M>0$. Let $\calL$ be a linear map given by $\calL(\bmx)=2M(\bmx-1/2)$. By Theorem \ref{thm:main},  for any $f\in C([-M,M]^d)$, there exists $\phi$  implemented by a Floor-ReLU network with  width $\max\{d,\,5N+13\}$ and  depth $64dL+3$ such that
\begin{equation*}
|\phi(\bmx)-f\circ \calL(\bmx)|\le \omega_{f\circ \calL}(\sqrt{d}\,N^{-\sqrt{L}})+2\omega_{f\circ \calL}(\sqrt{d}){N^{-\sqrt{L}}},\quad  \tn{for any $\bmx\in [0,1]^d$.}
\end{equation*}
It follows from $\bmy=\calL(\bmx)\in [-M,M]^d$ and $\omega_{f\circ \calL}{(r)}=  \omega_f^{\scriptscriptstyle[-M,M]^d}(2Mr)$ for any $r\ge 0$ that,\footnote{For an arbitrary set $E\subset\R^d$, $\omega_f^E(r)$ is defined via $\omega_f^E(r)\coloneqq  \sup\big\{|f(\bmx)-f(\bmy)|: \|\bmx-\bmy\|_2\le r,\ \bmx,\bmy\in E\big\},$ for any $r\ge0$. As defined earlier, $\omega_f(r)$ is short of $\omega_f^{[0,1]^d}(r)$. }  \tn{for any $\bmy\in [-M,M]^d$,}
\begin{equation}\label{eqn:Mapp}
    |\phi(\tfrac{\bmy+M}{2M})-f(\bmy)|\le \omega_f^{\scriptscriptstyle[-M,M]^d}(2M\sqrt{d}\,N^{-\sqrt{L}})+2\omega_f^{\scriptscriptstyle[-M,M]^d}(2M\sqrt{d}){N^{-\sqrt{L}}}.
\end{equation}
Hence, the size of the function domain $[-M,M]^d$ only has a mild influence on the approximation rate of our Floor-ReLU networks. Floor-ReLU networks can still avoid the curse of dimensionality and achieve root exponential convergence for continuous functions on $[-M,M]^d$ when $M>1$.  For example, in the case of H{\"o}lder continuous functions of order $\alpha$ with a constant $\lambda$ on $[-M,M]^d$, our approximation rate becomes $3\lambda (2M\sqrt{d}{N^{-\sqrt{L}}})^\alpha$.
}

\sj{Second, most interesting continuous functions in practice have a good modulus of continuity such that there is no implicit curse of dimensionality hiding in $\omega_f(\cdot)$. For example, we have discussed the case of H{\"o}lder continuous functions previously. We would like to remark that the class of H{\"o}lder continuous functions implicitly depends on $d$ through its definition in \eqref{eqn:Holder}, but this dependence is moderate since the $\ell^2$-
norm in \eqref{eqn:Holder} is the square root of a sum with $d$ terms.  Let us now discuss several cases of $\omega_f(\cdot)$ when we cannot achieve exponential convergence or cannot avoid the curse of dimensionality. The first example is $\omega_f(r)=\tfrac{1}{\ln (1/r)}$ for {all} small $r> 0$, which leads to an approximation rate
\begin{equation*}
  3(\sqrt{L}\ln N-\tfrac12\ln d)^{-1},\quad \tn{for large $N,L\in\N^+$}.
\end{equation*} 
Apparently, the above approximation rate still avoids the curse of dimensionality but there is no exponential convergence, which has been canceled out by ``$\ln$'' in $\omega_f(\cdot)$. The second example is $\omega_f(r)=\tfrac{1}{\ln^{1/d} (1/r)}$ for {all} small $r> 0$, which leads to an approximation rate 
\begin{equation*}
 3(\sqrt{L}\ln N-\tfrac12\ln d)^{-1/d},\quad \tn{for large $N,L\in\N^+$}.
\end{equation*} 
The power $\frac{1}{d}$ further weaken{s} the approximation rate and hence the curse of dimensionality {occurs}. The last example we would like to discuss is $\omega_f(r)=r^{\alpha/d}$ for {all} small $r> 0$, which results in the approximation rate
\[
3d^{\tfrac{\alpha}{2d}}{N^{-\tfrac{\alpha}{d}\sqrt{L}}},\quad \tn{for large $N,L\in\N^+$}, 
\]
which achieves the exponential convergence and avoids the curse of dimensionality when we use very deep networks with a fixed width. But if we fix {the} depth, there is no exponential convergence and the curse {occurs}. Though we have provided several examples of immoderate $\omega_f(\cdot)$, to the best of our knowledge, we are not aware of {practically} useful continuous functions with $\omega_f(\cdot)$ that is immoderate.}

\subsection{Discussion on the literature}

{The neural network{s} constructed here achieve exponential convergence without the curse of dimensionality simultaneously for a function class as general as (H\"older) continuous functions, while--to the best of our knowledge--most existing theories only apply to  functions with an intrinsic low complexity.} For example, the exponential convergence was studied for polynomials \cite{yarotsky2017,bandlimit,Shen3}, smooth functions \cite{bandlimit,DBLP:journals/corr/LiangS16}, analytic functions \cite{DBLP:journals/corr/abs-1807-00297}, {and} functions admitting a holomorphic extension to a Bernstein polyellipse \cite{Opschoor2019}. For another example, {no curse of dimensionality occurs, or the curse is lessened for Barron
spaces} \cite{barron1993,Weinan2019,e2020representation}, Korobov spaces \cite{Hadrien}, band-limited functions \cite{doi:10.1002/mma.5575,bandlimit}, compositional functions \cite{poggio2017}, and smooth functions \cite{yarotsky2019,Shen3,MO,Wang}. 

\sj{Our theory admits a neat and explicit approximation error {bound}. For example, our approximation rate in the case of H{\"o}lder continuous functions of order $\alpha$ with a constant $\lambda$ is $3\lambda d^{\alpha/2}{N^{-\alpha\sqrt{L}}}$,} while the prefactor of most existing theories is unknown or grows exponentially in $d$.  Our proof fully explores the advantage of the compositional structure and the nonlinearity of deep networks, while many existing theories were built on traditional approximation tools (e.g., polynomial approximation, multiresolution analysis, and Monte Carlo sampling), \sj{making it challenging for existing theories to obtain a neat and explicit error {bound} with an exponential convergence and without the curse of dimensionality.} 

Let us review existing works in more detail below.

\textbf{Curse of dimensionality.} \sj{The curse of dimensionality is the {phenomenon} that approximating a $d$-dimensional function using a certain parametrization method with a fixed target accuracy generally requires a large number of parameters that is exponential in $d$ and this expense quickly becomes unaffordable when $d$ {is large}. For example, traditional finite element methods with $W$ parameters can achieve an approximation accuracy $O(W^{-1/d})$ with an explicit indicator of the curse $\frac{1}{d}$ in the power of $W$. If an approximation rate has a constant independent of $W$ and exponential in $d$, the curse still {occurs} implicitly through this prefactor by definition. If the approximation rate has a prefactor $C_f$ depending on $f$, then the prefactor $C_f$ still depends on $d$ implicitly via $f$ and the curse implicitly {occurs} if $C_f$ exponentially grows when $d$ increases. Designing a parametrization method that can overcome the curse of dimensionality is an important research topic in approximation theory.}

In \cite{barron1993} and its variants or generalization \cite{Weinan2019,e2020representation,doi:10.1002/mma.5575,bandlimit},  $d$-dimensional functions \sj{defined on a domain $\Omega\subset\mathbb{R}^d$ admitting an integral representation with an integrand as a ridge function on $\widetilde{\Omega}\subset\mathbb{R}^d$ with a variable coefficent were considered, e.g.,
\begin{equation}\label{eqn:fint}
f(\bmx)=\int_{\widetilde{\Omega}} a(\bm{w}) K(\bm{w}\cdot\bmx)d\nu(\bm{w}),
\end{equation}
where {$\nu(\bm{w})$} is a Lebesgue measure in $\bm{w}$. $f(\bmx)$ can be reformulated into the expectation of a high-dimensional random function when $\bm{w}$ is treated as a random variable. Then  $f(\bmx)$ can be approximated by the average of $W$ samples of the integran{d} in the same spirit of the law of large numbers with an approximation error essentially bounded by $\frac{C_f \sqrt{\mu(\Omega)}}{\sqrt{W}}$ measured in $L^2(\Omega,\mu)$ (Equation (6) of \cite{barron1993}), where $\calO(W)$ is the total number of parameters in the network, $C_f$ is a $d$-dimensional integral with an integran{d} related to $f$, and $\mu(\Omega)$ is the Lebesgue measure of $\Omega$. As pointed out in \cite{barron1993} right after Equation (6), if $\Omega$ is not a unit domain in $\mathbb{R}^d$, $\mu(\Omega)$ would be exponential in $d$; at the beginning of Page 932 of \cite{barron1993}, it was remarked that $C_f$ can often be exponentially large in $d$ and standard smoothness properties of $f$ alone are not enough to remove the exponential dependence of $C_f$ on $d$, though there is a large number of examples for which $C_f$ is only moderately large. Therefore, the curse of dimensionality {occurs} unless $C_f$ and $\mu(\Omega)$ are not exponential in $d$. It was observed that if the error is measured in the sense of mean squared error in machine learning, which is the square of the $L^2(\Omega,\mu)$ error averaged over $\mu(\Omega)$ resulting in $\frac{C^2_f }{W}$, then the mean squared error has no curse of dimensionality as long as $C_f$ is not exponential in $d$ \cite{barron1993,Weinan2019,e2020representation}. }

In \cite{Hadrien}, $d$-dimensional functions in the Korobov space are approximated by the linear combination of basis functions of a sparse grid, each of which is approximated by a ReLU network. \sj{Though the curse of dimensionality has been lessened, target functions have to be sufficiently smooth and the approximation error still contains a factor that is exponential in $d$, i.e., the curse still {occurs}.}  Other works in \cite{yarotsky2017,yarotsky2019,Shen3,Wang} \sj{study the advantage of smoothness in the network approximation. Polynomials are applied to approximate smooth functions and ReLU networks are constructed to approximate polynomials. The application of smoothness can lessen the curse of dimensionality in the approximation rates in terms of network sizes but also results in a prefactor that is exponentially large in the dimension, which means that the curse still {occurs} implicitly.} 

\sj{The Kolmogorov-Arnold
superposition theorem (KST) \cite{kolmogorov1956,arnold1957,kolmogorov1957} has also inspired a research direction of network approximation  \cite{kurkova1992,MAIOROV199981,igelnik2003,MO} {for continuous functions.} \cite{kurkova1992} provided a quantitative approximation rate of networks with two hidden layers, but the number of neurons scales exponentially in the dimension and the curse {occurs}. \cite{MAIOROV199981} relaxes the exact representation in KST to an approximation in a form of two-hidden-layer neural networks with a maximum width $6d+3$ and a single activation function. {This powerful activation function is very complex as described by its authors and its numerical evaluation was not available until a more concrete algorithm was recently proposed in \cite{GUL}. Note that there is no available numerical algorithm in \cite{MAIOROV199981,GUL} to compute the whole networks proposed therein. The difficulty is due to the fact that the construction of these networks relies on the outer univariate continuous function of the KST. Though the existence of these outer functions can be shown by construction via a complicated iterative procedure in \cite{braun}, there is no existing numerical algorithm to evaluate them for a given target function yet, even though computation with an arbitrary precision is assumed to be available. Therefore, the networks considered in \cite{MAIOROV199981,GUL} are similar to the original representation in KST in the sense that their existence is proved without an explicit way or numerical algorithm to construct them.} \cite{igelnik2003} and \cite{MO} apply cubic-splines and piecewise linear functions to approximate the inner and outer functions of KST, resulting in cubic-spline and ReLU networks to approximate continuous functions on $[0,1]^d$. Due to the pathological outer functions of KST, the approximation bounds still suffer from the curse of dimensionality unless target functions are restricted to a small class of functions with simple outer functions in the KST.

Recently in \cite{yarotsky2019}, Sine-ReLU networks {have been} applied to approximate H{\"o}lder continuous functions of order $\alpha$ on $[0,1]^d$ with an approximation accuracy $\epsilon=\text{exp}(-c_{\alpha,d}W^{1/2})$, where $W$ is the number of parameters in the network and $c_{\alpha,d}$ is a positive constant depending on $\alpha$ and $d$ only. Whether or not $c_{\alpha,d}$ exponentially depends on $d$ determines whether or not the curse of dimensionality exists for the Sine-ReLU networks, which is not answered in \cite{yarotsky2019} and is still an open question. } 

\sj{Finally, we would like to discuss the curse of dimensionality in terms of the continuity of the weight selection as a map $\Sigma:C([0,1]^d)\rightarrow \mathbb{R}^W$. For a fixed network architecture with a fixed number of parameters $W$, let $g:\mathbb{R}^W\rightarrow C([0,1]^d)$ be the map of realizing a DNN from a given set of parameters in $\mathbb{R}^W$ to a function in $C([0,1]^d)$. Suppose that there is a continuous map $\Sigma$ from the unit ball of Sobolev space with smoothness $s$, denoted as $F_{s,d}$, to $\mathbb{R}^W$ such that $\|f-g(\Sigma(f))\|_{L^\infty}\leq \epsilon$ for all $f\in F_{s,d}$. Then $W\geq c \epsilon^{-d/s}$ with some constant $c$ depending only on $s$. This conclusion is given in Theorem 3 of \cite{yarotsky2017}, which is a corollary of Theorem 4.2 of \cite{Devore89optimalnonlinear} in a more general form.  Intuitively, this conclusion means that any constructive approximation of ReLU FNNs to approximate $C([0,1]^d)$ cannot enjoy a continuous weight selection property if the approximation rate is better than $c \epsilon^{-d/s}$, i.e., the curse of dimensionality must {occur} for constructive approximation for ReLU FNNs with a continuous weight selection. {Theorem 4.2 of \cite{Devore89optimalnonlinear} can also lead to a new corollary with a weight selection map $\Sigma:K_{s,d}\rightarrow \mathbb{R}^W$ (e.g., the constructive approximation of Floor-ReLU networks) and $g:\mathbb{R}^W\rightarrow L^\infty([0,1]^d)$ (e.g., the realization map of Floor-ReLU networks), where $K_{s,d}$ is the unit ball of $C^s([0,1]^d)$ with the Sobolev norm $W^{s,\infty}([0,1]^d)$. Then this new corollary implies that the constructive approximation in this paper cannot enjoy continuous weight selection.} 
However, Theorem 4.2 of \cite{Devore89optimalnonlinear} is essentially a min-max criterion to evaluate weight selection maps maintaining continuity: the approximation error obtained by minimizing over all continuous selection $\Sigma$ and network realization $g$ and maximizing over all target functions is bounded below by $\calO(W^{-s/d})$. In the worst scenario, a continuous weight selection cannot enjoy an approximation rate beating the curse of dimensionality. However, Theorem 4.2 of \cite{Devore89optimalnonlinear} has not excluded the possibility that most continuous functions of interest in practice may still enjoy a continuous weight selection without the curse of dimensionality. }

\textbf{Exponential convergence.} \sj{Exponential convergence is referred to as the situation that the approximation error exponentially decays to zero when the number of parameters increases. Designing approximation tools with an exponential convergence is another important topic in approximation theory. In the literature of deep network approximation, when the number of network parameters $W$ is a polynomial of $\calO(\log(\frac{1}{\epsilon}))$, the terminology ``exponential convergence" was also used \cite{DBLP:journals/corr/abs-1807-00297,yarotsky2019,Opschoor2019}. The exponential convergence in this paper is root-exponential as in \cite{yarotsky2019}, i.e., $W=\calO(\log^2(\frac{1}{\epsilon}))$. The exponential convergence in other works is worse than root-exponential.}

In most cases, the approximation power to achieve exponential approximation rates in existing works comes from traditional tools for approximating a small class of functions instead of taking advantage of the network structure itself. In \cite{DBLP:journals/corr/abs-1807-00297,Opschoor2019}, highly smooth functions are first approximated by the linear combination of special polynomials with high degrees (e.g., Chebyshev polynomials, Legendre polynomials) with an exponential approximation rate, i.e., to achieve an $\epsilon$-accuracy, a linear combination of only $\calO(p(\log(\frac{1}{\epsilon})))$ polynomials is required, where $p$ is a polynomial with a degree that may depend on the dimension $d$. Then each polynomial is approximated by a ReLU network with $\calO(\log(\frac{1}{\epsilon}))$ parameters. Finally, all ReLU networks are assembled to form a large network approximating the target function with an exponential approximation rate. As far as we know, the only existing work that achieves exponential convergence without taking advantage of special polynomials and smoothness is the Sine-ReLU network in \cite{yarotsky2019}, which has been mentioned in the paragraph just above. {We would like to emphasize that the result in our paper applies for generic continuous functions including, but not limited to, the H{\"o}lder continuous functions considered in \cite{yarotsky2019}.}

\section{Approximation of continuous functions}
\label{sec:approxContFunc}
In this section, we first introduce basic notations in this paper in Section \ref{sec:notation}. Then we prove Theorem \ref{thm:main} \sj{based on Proposition \ref{prop:bitsExtraction}, which will be proved in Section \ref{sec:proofOfBitsExtraction}.}

\subsection{Notations}
\label{sec:notation}

The main notations of this paper are listed as follows.
\begin{itemize}    
\item {Vectors and matrices are denoted in a bold font. Standard vectorization is adopted in the matrix and vector computation. For example, {adding a scalar and a vector} means adding the scalar to each entry of the vector.}

    \item Let $\N^+$ denote the set containing all positive integers, i.e., $\N^+=\{1,2,3,\cdots\}$.    
    
    \item Let $\sigma:\R\to \R$ denote the rectified linear unit (ReLU), i.e. $\sigma(x)=\max\{0,x\}$.  With {a slight} abuse of notation, we define $\sigma:\R^d\to \R^d$ as $\sigma(\bmx)=\left[\begin{array}{c}
          \max\{0,x_1\}  \\
          \vdots \\
          \max\{0,x_d\}
     \end{array}\right]$ for any $\bmx=(x_1,\cdots,x_d)\in \R^d$.

    \item The floor function (Floor) is defined as $\lfloor x\rfloor:=\max \{n: n\le x,\ n\in \mathbb{Z}\}$ for any $x\in \R$. 
    
    \item For $\theta\in[0,1)$, suppose its binary representation is $\theta=\sum_{\ell=1}^{\infty}\theta_\ell2^{-\ell}$ with $\theta_\ell\in \{0,1\}$, we introduce a special notation $\bin 0.\theta_1\theta_2\cdots \theta_L$ to denote the $L$-term binary representation of $\theta$, i.e., ${\bin 0.\theta_1\theta_2\cdots \theta_L\coloneqq}\sum_{\ell=1}^{L}\theta_\ell2^{-\ell}$.

    \item The expression ``a network with  width $N$ and  depth $L$'' means
    \begin{itemize}
        \item The maximum width of this network for all \textbf{hidden} layers  is no more than $N$.
        \item The number of \textbf{hidden} layers of this network is  no more than $L$.
    \end{itemize}
\end{itemize}

\subsection{Proof of Theorem \ref{thm:main}}
\label{sec:proofOfMain}

{Theorem \ref{thm:main} is an immediate consequence of Theorem \ref{thm:mainOld} below.}
\begin{theorem}
    \label{thm:mainOld}
    Given any $N,L\in \N^+$ and an arbitrary continuous function $f$ on $[0,1]^d$, there exists a function $\phi$ implemented by a Floor-ReLU network  with  width $\max\{d,\,2N^2+5N\}$ and  depth $7dL^2+3$ such that
    \begin{equation*}
    |\phi(\bmx)-f(\bmx)|\le \omega_f(\sqrt{d}\,N^{-L})+2\omega_f(\sqrt{d})2^{-NL},\quad \tn{for any $\bmx\in [0,1]^d$.}
    \end{equation*}
\end{theorem}
This theorem will be proved later in this section. Now let us prove Theorem \ref{thm:main} based on Theorem \ref{thm:mainOld}. 
\begin{proof}[Proof of Theorem \ref{thm:main}]
    Given any $N,L\in \N^+$, there exist $\tildeN,\tildeL\in \N^+$ with $\tildeN\ge 2$ and $\tildeL\ge 3$ such that
    \begin{equation*}
    (\tildeN-1)^2\le N< \tildeN^2\quad \tn{and} \quad (\tildeL-1)^2\le 4L< \tildeL^2.
    \end{equation*}
    By Theorem \ref{thm:mainOld}, there exists a function $\phi$ implemented by a Floor-ReLU network  with width $\max\{d,\, 2\tildeN^2+5\tildeN\}$ and  depth $7d\tildeL^2+3$ such that 
    \begin{equation*}
    |\phi(\bmx)-f(\bmx)|\le \omega_f(\sqrt{d}\,\tildeN^{-\tildeL})+2\omega_f(\sqrt{d})2^{-\tildeN\tildeL},\quad \tn{for any $\bmx\in [0,1]^d$.}
    \end{equation*}
    Note that \[2^{-\tildeN\tildeL}\le \tildeN^{-\tildeL}=(\tildeN^2)^{-\tfrac12\sqrt{\tildeL^2}}\le N^{-\tfrac12 \sqrt{4L}}\le N^{-\sqrt{L}}.\]
    Then we have
    \begin{equation*}
    |\phi(\bmx)-f(\bmx)|\le \omega_f(\sqrt{d}\,N^{-\sqrt{L}})+2\omega_f(\sqrt{d})N^{-\sqrt{L}},\quad \tn{for any $\bmx\in [0,1]^d$.}
    \end{equation*}
    
    For $\tildeN,\tildeL\in \N^+$ with $\tildeN\ge 2$ and $\tildeL\ge 3$, we have
    \begin{equation*}
    \sj{2\tildeN^2+{5}\tildeN\le 5(\tildeN-1)^2+13\le 5N+13\quad \tn{and} \quad 7\tildeL^2\le 16(\tildeL-1)^2\le 64L.}
    \end{equation*}
    Therefore, $\phi$ can be computed by a Floor-ReLU network with  width $\max\{d,\, 2\tildeN^2+5\tildeN\}\le \max\{d,\, 5N+13\}$ and  depth $7d\tildeL^2+3\le 64dL+3$, as desired. So we finish the proof.
\end{proof}

To prove Theorem \ref{thm:mainOld}, we first present the proof sketch.
{Put briefly}, we construct piecewise constant functions implemented by Floor-ReLU networks to approximate continuous functions. There are four key steps in our construction.
\begin{enumerate}
		\item Normalize $f$ as $\tildef$ satisfying $\tildef(\bmx)\in[0,1]$ for any $\bmx\in [0,1]^d$, divide $[0,1]^d$ into a set of non-overlapping cubes $\{Q_\bmbeta\}_{\bmbeta\in \{0,1,\cdots,K-1\}^d}$, and denote $\bmx_\bmbeta$ as the vertex of $Q_\bmbeta$ with minimum $\|\cdot\|_1$ norm,  where $K$ is an integer determined later. See Figure \ref{fig:Qbeta+xbeta} for the illustrations of $Q_\bmbeta$ and $\bmx_\bmbeta$.
		
		\item Construct a Floor-ReLU sub-network to implement a vector-valued function $\bm{\Phi}_1:\R^d\to \R^d$ projecting  the whole cube $ Q_\bmbeta$ to the index $\bmbeta$ for each $\bmbeta\in \{0,1,\cdots,K-1\}^d$, i.e., $\bmPhi_1(\bmx)=\bmbeta$ for all $\bmx\in Q_\bmbeta$.

		\item Construct a Floor-ReLU sub-network to implement a function $\phi_2:\R^d\to \R$ mapping $\bmbeta\in\{0,1,\cdots,K-1\}^d$ approximately to $\tildef(\bmx_\bmbeta)$ for each $\bmbeta$, i.e., $\phi_2(\bmbeta)\approx \tildef(\bmx_\bmbeta)$. Then $\phi_2\circ\bm{\Phi}_1(\bmx)=\phi_2(\bmbeta)\approx \tildef(\bmx_\bmbeta)$ for any $\bmx\in Q_\bmbeta$ and each $\bmbeta\in \{0,1,\cdots,K-1\}^d$, implying $\tildephi\coloneqq\phi_2\circ\bm{\Phi}_1$ approximates $\tildef$ within an error $\calO(\omega_f(1/K))$ on $[0,1]^d$.

		\item Re-scale and shift $\tildephi$ to obtain the desired function $\phi$ approximating $f$ well and determine the final Floor-ReLU network to implement $\phi$.
\end{enumerate}

It is not difficult to construct Floor-ReLU networks with the desired width and depth to implement $\bm{\Phi}_1$. The most technical part is the construction of a Floor-ReLU network with the desired width and depth computing $\phi_2$, which needs the following proposition based on the ``bit extraction'' technique introduced in \cite{Bartlett98almostlinear,pmlr-v65-harvey17a}. 

\begin{proposition}
    \label{prop:bitsExtraction}
    Given any $N,L\in \N^+$ and arbitrary $\theta_m\in \{0,1\}$ for $m=1,2,\cdots,N^L$, there exists a function $\phi$ computed by  a Floor-ReLU network   with width $2N+2$ and  depth $7L-2$ such that
    \begin{equation*}
    \phi(m)=\theta_m,\quad \tn{for $m= {1,2,\cdots,N^L}$.}
    \end{equation*}
\end{proposition}

The proof of this proposition is presented in Section \ref{sec:proofOfBitsExtraction}.
By this proposition and the definition of VC-dimension (e.g., see \cite{pmlr-v65-harvey17a}), it is easy to prove that the VC-dimension of Floor-ReLU networks with a constant width and  depth $\calO(L)$ has a lower bound $2^L$. Such a lower bound is much larger than $\calO(L^2)$, which is a VC-dimension upper bound of ReLU networks with the same width and depth due to Theorem 8 of \cite{pmlr-v65-harvey17a}. This means Floor-ReLU networks are much more powerful than ReLU networks from the perspective of VC-dimension.


\sj{Based on the proof sketch stated just above, we are ready to give the detailed proof of Theorem \ref{thm:mainOld} following similar ideas {as} in our previous work \cite{Shen1,Shen2,Shen3}. The main idea of our proof is to reduce high-dimensional approximation to one-dimensional approximation via a projection. The idea of projection was probably first used in well-established theories, e.g., KST {(Kolmogorov superposition
theorem)} mentioned in Section \ref{sec:dis}, where the approximant to high-dimensional functions is constructed by: first, projecting high-dimensional data points to one-dimensional data points; second, construct one-dimensional approximants. There has been extensive research based on this idea, e.g., references related to KST summarized in Section \ref{sec:dis}, our previous works \cite{Shen1,Shen2,Shen3}, and \cite{yarotsky2019}. The key to a successful approximant is to construct one-dimensional approximants to deal {with a large number of one-dimensional data points; in fact, the number
of points is exponential in the dimension $d$.} }

\begin{proof}[Proof of Theorem \ref{thm:mainOld}]
		The proof consists of four steps.
	\mystep{1}{Set up.}
	Assume $f$ is not a constant function since it is a trivial case. Then $\omega_f(r)>0$ for any $r>0$. Clearly, $|f(\bmx)-f(\bmzero)|\le \omega_f(\sqrt{d})$ for any $\bmx\in [0,1]^d$. Define 
	\begin{equation}
	\label{eq:tildef}
	\tildef\coloneqq\big(f-f(\bmzero)+\omega_f(\sqrt{d})\big)\big/\big(2\omega_f(\sqrt{d})\big).
	\end{equation} 
	It follows that $\tildef(\bmx)\in [0,1]$ for any $\bmx\in [0,1]^d$.     
	
	Set $K= N^L$, $E_{K-1}=[\tfrac{K-1}{K},1]$, and $E_k=[\tfrac{k}{K},\tfrac{k+1}{K})$ for $k=0,1,\cdots,K-2$. 
	Define $\bmx_\bmbeta\coloneqq \bmbeta/K$ and 
	\begin{equation*}
	Q_\bmbeta\coloneqq \Big\{\bmx=(x_1,x_2,\cdots,x_d)\in \R^d: x_j\in E_{\beta_j}\ \tn{for}\ j=1,2,\cdots,d\Big\},
	\end{equation*}
	for any $\bmbeta=(\beta_1,\beta_2\cdots,\beta_d)\in \{0,1,\cdots,K-1\}^d$. See Figure \ref{fig:Qbeta+xbeta} for the examples of $Q_\bmbeta$ and $\bmx_\bmbeta$ for $\bmbeta\in \{0,1,\cdots,K-1\}^d$ with $K=4$ and $d=1,2$.
	
	\begin{figure}[!htp]
		\centering
		\begin{subfigure}[b]{0.362014\textwidth}
			\centering            \includegraphics[width=0.758\textwidth]{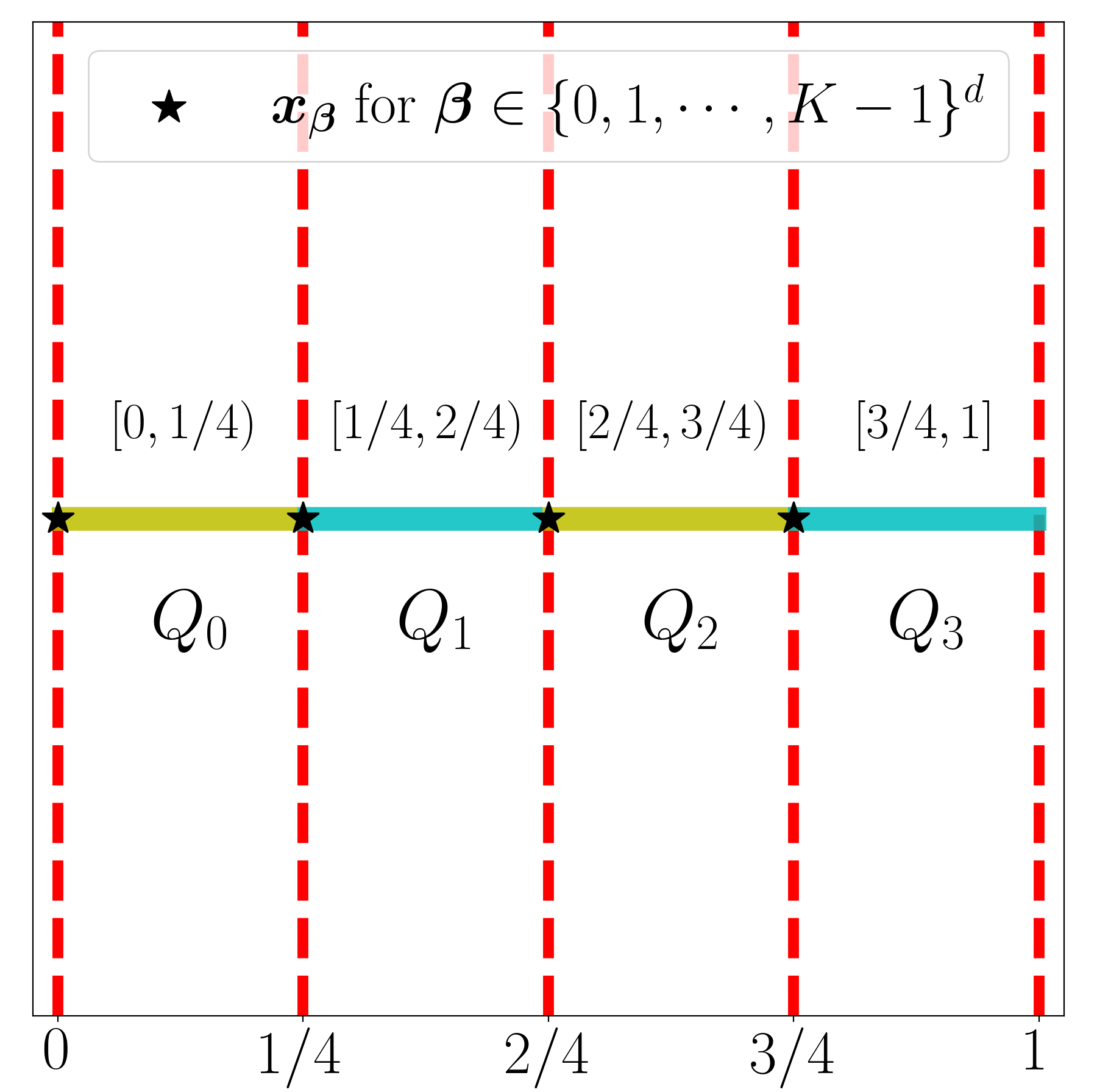}
			\subcaption{}
		\end{subfigure}
		\begin{subfigure}[b]{0.362014\textwidth}
			\centering            \includegraphics[width=0.758\textwidth]{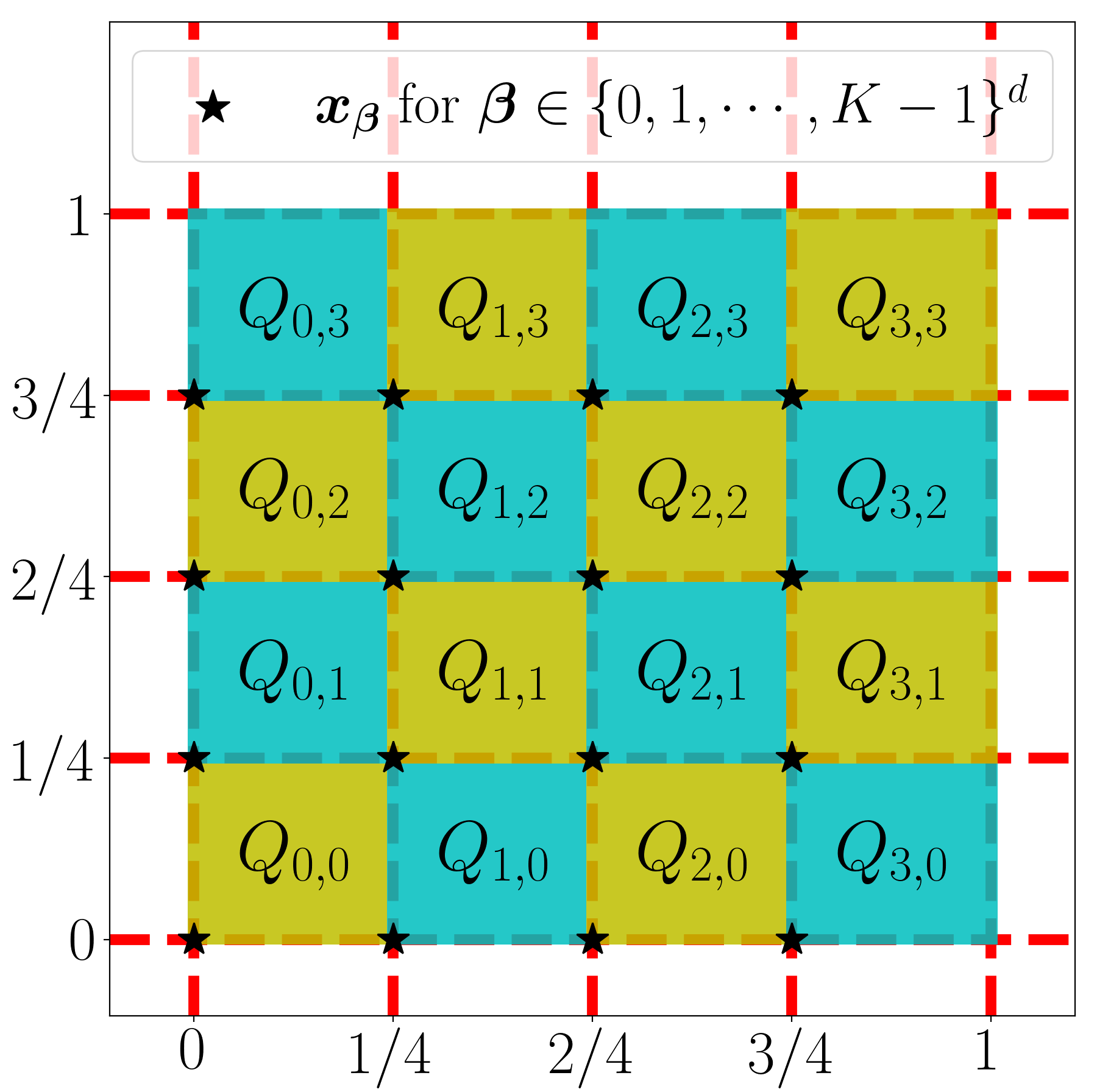}
			\subcaption{}
		\end{subfigure}
		\caption[Illustrations of $Q_\bmbeta$  and $\bmx_\bmbeta$ for $\bmbeta\in \{0,1,\cdots,K-1\}^d$]{Illustrations of $Q_\bmbeta$ and $\bmx_\bmbeta$ for $\bmbeta\in \{0,1,\cdots,K-1\}^d$. (a) $K=4,\ d=1$. (b)  $K=4,\ d=2$.}
		\label{fig:Qbeta+xbeta}
	\end{figure}

	\mystep{2}{Construct $\bm{\Phi}_1$ mapping  $\bmx\in Q_\bmbeta$ to $\bmbeta$.}
	Define a step function $\phi_1$ as
	\begin{equation*}
	\phi_1(x)\coloneqq \big\lfloor -\sigma(-Kx+K-1)+K-1\big\rfloor,\quad \tn{for any $x\in \R$.} \footnote{If we just define $\phi_1(x)=\lfloor Kx\rfloor$, then $\phi_1(1)=K\neq K-1$ even though $1\in E_{K-1}$.}
	\end{equation*}
	See Figure \ref{fig:phiOne} for an example of $\phi_1$ when $K=4$.
	It follows from the definition of $\phi_1$ that
	\begin{equation*}
	\phi_1(x)=k,\quad \tn{if $x\in E_k$,\ for $k=0,1,\cdots,K-1$.}
	\end{equation*}
	
	\begin{figure}[!htp]
		\centering
		\includegraphics[width=0.805\textwidth]{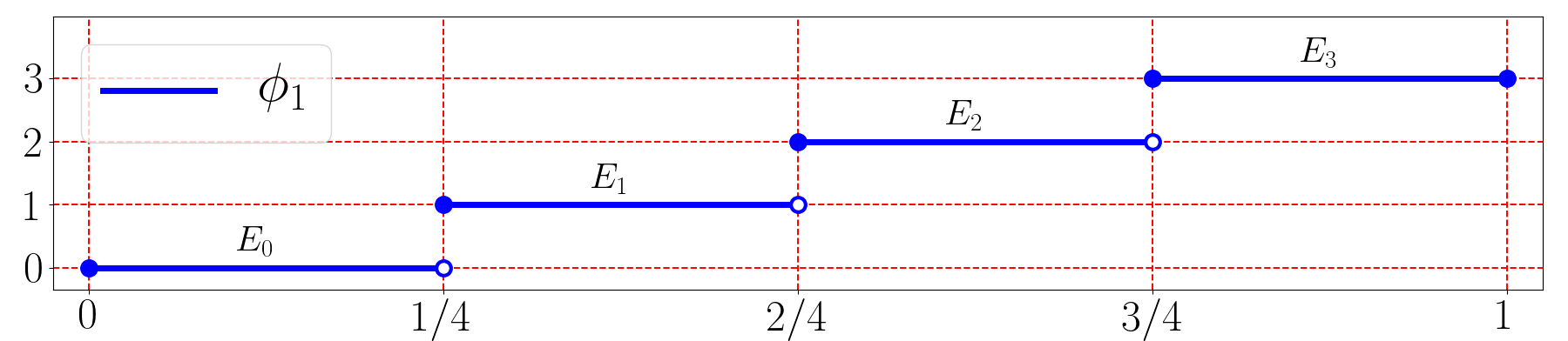}
		\caption{An illustration of $\phi_1$ on $[0,1]$ for the case $K=4$.}
		\label{fig:phiOne}
	\end{figure}	
	Define 
	\begin{equation*}
	\bm{\Phi}_1(\bmx)\coloneqq\big(\phi_1(x_1),\phi_1(x_2),\cdots,\phi_1(x_d)\big),\quad \tn{for any $\bmx=(x_1,x_2,\cdots,x_d)\in\R^d$.}
	\end{equation*}
	Clearly,  we have, for $\bmx\in Q_\bmbeta$ and $\bmbeta\in \{0,1,\cdots,K-1\}^d$,
	\begin{equation*}
	\bm{\Phi}_1(\bmx)=\big(\phi_1(x_1),\phi_1(x_2),\cdots,\phi_1(x_d)\big)=(\beta_1,\beta_2,\cdots,\beta_d)=\bmbeta.
	\end{equation*}
	
	\mystep{3}{Construct  $\phi_2$  mapping $\bmbeta\in \{0,1,\cdots,K-1\}^d$ approximately to $\tildef(\bmx_\bmbeta)$.}
	Using the idea of $K$-ary representation, we 
	define  a linear function $\psi_1$
	via 
	\begin{equation*}
	\psi_1(\bmx)\coloneqq 1+\sum_{j=1}^d x_j K^{j-1},\quad \tn{for any $\bmx=(x_1,x_2,\cdots,x_d)\in \R^d$.}
	\end{equation*}
	Then $\psi_1$ is a bijection 
	from $\{0,1,\cdots,K-1\}^d$ to  $\{1,2,\cdots,K^d\}$.
	
	Given any $i\in \{1,2,\cdots,K^d\}$, there exists a unique $\bmbeta\in \{0,1,\cdots,K-1\}^d$ such that $i=\psi_1(\bmbeta)$. Then define
	\begin{equation*}
	\xi_i\coloneqq\tildef(\bmx_\bmbeta)\in [0,1],
	\quad \tn{for $i=\psi_1(\bmbeta)$ and $\bmbeta\in \{0,1,\cdots,K-1\}^d$,}
	\end{equation*}
	where $\tildef$ is the normalization of $f$ defined in Equation \eqref{eq:tildef}.
	It follows that there exists $\xi_{i,j}\in \{0,1\}$ for $j=1,2,\cdots,NL$ such that
	\begin{equation*}
	|\xi_i-\bin 0.\xi_{i,1}\xi_{i,2}\cdots,\xi_{i,NL}|\le 2^{-NL}, \quad \tn{for $i=1,2,\cdots,K^d$.}
	\end{equation*}
	
	By  $K^d=(N^L)^d=N^{dL}$ and Proposition \ref{prop:bitsExtraction}, there exists a function $\psi_{2,j}$ implemented by a Floor-ReLU network  with width $2N+2$ and depth $7dL-2$, for each $j=1,2,\cdots,NL$, such that
	\begin{equation*}
	\psi_{2,j}(i)=\xi_{i,j},\quad \tn{for $i=1,2,\cdots,K^d$.}
	\end{equation*}
	Define 
	\begin{equation*}
	    \psi_2\coloneqq \sum_{j=1}^{NL} 2^{-j}\psi_{2,j}\quad \tn{and}\quad  \phi_2\coloneqq \psi_2\circ\psi_1.
	\end{equation*}
	Then, for $i=\psi_1(\bmbeta)$ and $\bmbeta\in \{0,1,\cdots,K-1\}^d$, we have
	\begin{equation}
	\label{eq:2NL}
	\begin{split}
	|\tildef(\bmx_\bmbeta)-\phi_2(\bmbeta)|=|\tildef(\bmx_\bmbeta)-\psi_2(\psi_1(\bmbeta))|&=|\xi_i-\psi_2(i)|=|\xi_i- \sum_{j=1}^{NL} 2^{-j}\psi_{2,j}(i)|\\
	&=|\xi_i-\bin 0.\xi_{i,1}\xi_{i,2}\cdots\xi_{i,NL}|\le 2^{-NL}.
	\end{split}
	\end{equation}
	
	\mystep{4}{Determine the final network to implement the desired function $\phi$.}
	
	Define $\tildephi\coloneqq \phi_2\circ \bm{\Phi}_1$, i.e., for any $\bmx=(x_1,x_2,\cdots,x_d)\in \R^d$,
	\begin{equation*}
	\tildephi(\bmx)= \phi_2\circ\bm{\Phi}_1(\bmx)=\phi_2\big(\phi_1(x_1),\phi_1(x_2),\cdots,\phi_1(x_d)\big).
	\end{equation*}
	
	Note that $|\bmx-\bmx_\bmbeta|\le \tfrac{\sqrt{d}}{K}$ for any $\bmx\in Q_\bmbeta$ and $\bmbeta\in \{0,1,\cdots,K-1\}^d$.
	Then we have, for any $\bmx\in Q_\bmbeta$ and $\bmbeta\in \{0,1,\cdots,K-1\}^d$,
	\begin{equation*}
	\begin{split}
	|\tildef(\bmx)-\tildephi(\bmx)|
	&\le |\tildef(\bmx)-\tildef(\bmx_\bmbeta)|+|\tildef(\bmx_\bmbeta)-\widetilde{\phi}(\bmx)|\\
	&\le \omega_\tildef(\tfrac{\sqrt{d}}{K})+|\tildef(\bmx_\bmbeta)-\phi_2(\bm{\Phi}_1(\bmx))|\\
	&\le \omega_\tildef(\tfrac{\sqrt{d}}{K})+|\tildef(\bmx_\bmbeta)-\phi_2(\bmbeta)|\le \omega_\tildef(\tfrac{\sqrt{d}}{K})+2^{-NL},
	\end{split}
	\end{equation*}
	where the last inequality comes {from} Equation \eqref{eq:2NL}.
	
	Note that $\bmx\in Q_\bmbeta$ and $\bmbeta\in \{0,1,\cdots,K-1\}^d$ are arbitrary. 
	Since $[0,1]^d=\bigcup_{\bmbeta\in \{0,1,\cdots,K-1\}^d}Q_\bmbeta$, we have 
	\begin{equation*}
	|\tildef(\bmx)-\tildephi(\bmx)|\le \omega_\tildef(\tfrac{\sqrt{d}}{K})+2^{-NL},\quad \tn{for any $\bmx\in [0,1]^d$.}
	\end{equation*}

	Define \[\phi\coloneqq 2\omega_f(\sqrt{d})\tildephi+f(\bmzero)-\omega_f(\sqrt{d}).\]
	By $K=N^L$ and $\omega_f(r)=2\omega_f(\sqrt{d})\cdot\omega_\tildef(r)$ for any $r\ge 0$, we have, \tn{for any $\bmx\in [0,1]^d$,}
	\begin{equation*}
	\begin{split}
	|f(\bmx)-\phi(\bmx)|=2\omega_f(\sqrt{d})\big|\tildef(\bmx)-\tildephi(\bmx)\big|
	&\le 2\omega_f(\sqrt{d})\Big(\omega_\tildef(\tfrac{\sqrt{d}}{K})+2^{-NL}\Big)\\
	&\le \omega_f(\tfrac{\sqrt{d}}{K})+2\omega_f(\sqrt{d})2^{-NL}\\
	&\le \omega_f(\sqrt{d}\,N^{-L})+2\omega_f(\sqrt{d})2^{-NL}.
	\end{split} 
	\end{equation*}

	It remains  to determine the width and depth of the Floor-ReLU network implementing $\phi$.  Clearly, $\phi_2$ can be implemented by the architecture in Figure \ref{fig:phiTwo}. 
	\begin{figure}[!htp]        
		\centering
		\includegraphics[width=0.985\textwidth]{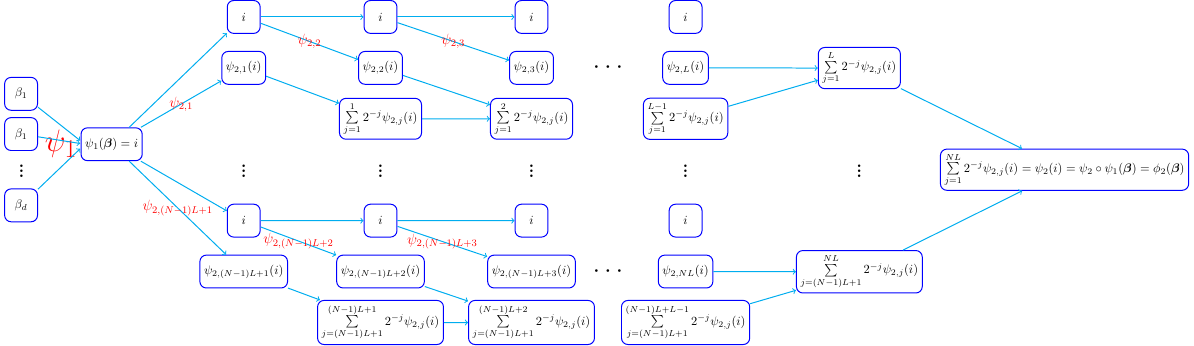}
		\caption[An illustration of  the desired network architecture implementing $\phi_2$]{An illustration of  the desired network architecture implementing $\phi_2=\psi_2\circ\psi_1$ for any input $\bmbeta\in \{0,1,\cdots,K-1\}^d$, where $i=\psi_1(\bmbeta)$.
		}
		\label{fig:phiTwo}
	\end{figure}
	
	As we can see from Figure \ref{fig:phiTwo}, $\phi_2$ can be implemented by a Floor-ReLU network  with width $N(2N+2+3)=2N^2+5N$ and depth $L(7dL-2 +1)+2= L(7dL-1)+2$. 
	With the network architecture implementing $\phi_2$ in hand, $\tildephi$ can be implemented by the network architecture shown in Figure \ref{fig:tildePhi}.	
	\begin{figure}[!htp]        
		\centering
		\includegraphics[width=0.8507705\textwidth]{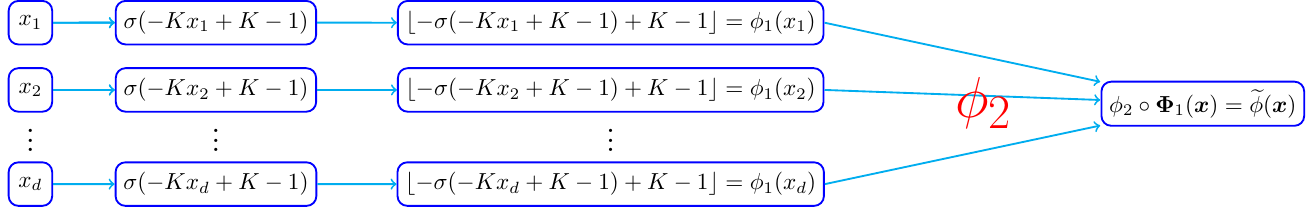}
		\caption{An illustration of  the network architecture implementing $\tildephi=\phi_2\circ\bmPhi_1$.}
		\label{fig:tildePhi}
	\end{figure}	
	Note that $\phi$ is defined via re-scaling and shifting $\tildephi$. 
	As shown in Figure \ref{fig:tildePhi}, $\phi$ and $\tildephi$ can be implemented by a Floor-ReLU network  with  width $\max\{d,\,2N^2+5N\}$ and depth $1+1+L(7dL-1)+2\le 7dL^2+3$. 
	So we finish the proof.

\end{proof}

\section{Proof of Proposition \ref{prop:bitsExtraction}}
\label{sec:proofOfBitsExtraction}

The proof of Proposition \ref{prop:bitsExtraction} mainly relies on the ``bit extraction'' technique. As we shall see later, our key idea is to apply the Floor activation function to make ``bit extraction'' more powerful to reduce network sizes. In particular, Floor-ReLU networks can extract much more bits than ReLU networks with the same network size.

Let us first establish a basic lemma to extract $1/N$ of the total bits of a binary number; the result is again stored in a binary number.

\begin{lemma}
    \label{lem:bitsExtractionBasic}
    Given any $J,N\in \N^+$,
    there exists a function $\phi:\R^2\to \R$ that can be implemented by a Floor-ReLU network   with width $2N$ and depth $4$ such that, for any $\theta_j\in \{0,1\}$, $j=1,\cdots,NJ$, we have
    \begin{equation*}
    \phi(\bin 0.\theta_1\cdots \theta_{NJ},\,n)=\bin 0.\theta_{(n-1)J+1}\cdots \theta_{nJ},\quad \tn{for $n=1,2,\cdots,N$}.
    \end{equation*}
\end{lemma}
\begin{proof}
    Given any $\theta_j\in \{0,1\}$ for $j=1,\cdots,NJ$, denote 
    \begin{equation*}
    s=\bin 0.\theta_{1}\cdots \theta_{NJ}\quad \tn{and}\quad  s_n=\bin 0.\theta_{(n-1)J+1}\cdots \theta_{nJ}, \quad \tn{for $n=1,2,\cdots,N$.}
    \end{equation*}
    
     Then our goal is to construct a function $\phi:\R^2\to \R$ computed by a Floor-ReLU network  with the desired width and depth  that satisfies 
     \begin{equation*}
     \phi(s,\, n)=s_n,\quad \tn{for $n=1,2,\cdots,N$.}
     \end{equation*}
    
    Based on the properties of the binary representation, it is easy to check that            
    \begin{equation}
    \label{eq:iterationFormula}
    s_n=\lfloor2^{nJ}s\rfloor\big/2^{\sj{J}}-\lfloor2^{(n-1)J}s\rfloor,\quad \tn{for $n=1,2,\cdots,N$}.
    \end{equation}    
    Even with the above formulas to generate $s_1,s_2,\cdots,s_N$, it is still technical to construct a network outputting $s_n$ for a given index $n\in \{1,2,\cdots,N\}$.
    
    Set $\delta=2^{-\sj{J}}$ and define $g$ (see Figure \ref{fig:activationFunG}) as     
    \begin{equation*}
    g(x)\coloneqq \sigma\big(\sigma(x)-\sigma(\tfrac{x+\delta-1}{\delta})\big), \quad \tn{where $\sigma(x)=\max\{0,x\}$}.
    \end{equation*}     
    \begin{figure}[H]
        \centering
        \includegraphics[width=0.66\textwidth]{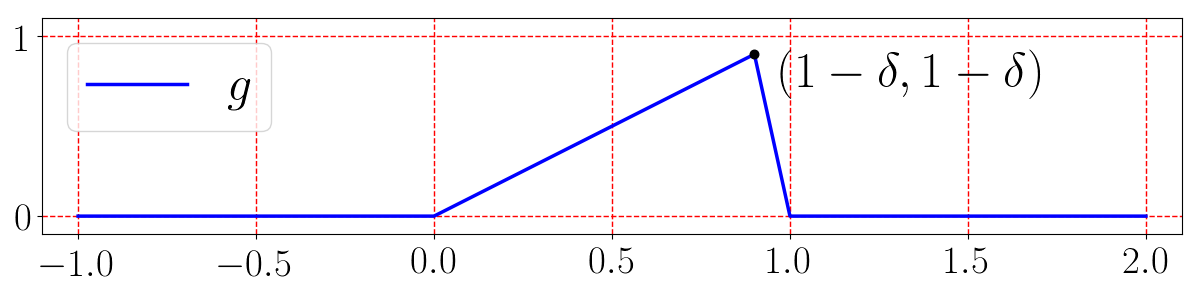}
        \caption{An illustration of $g(x)=\sigma\big(\sigma(x)-\sigma(\tfrac{x+\delta-1}{\delta})\big)$, where $\sigma(x)=\max\{0,x\}$ is the ReLU activation function.}
        \label{fig:activationFunG}
    \end{figure}
         
    Since $s_n\in [0,1-\delta]$ for $n=1,2,\cdots,N$,  we have
    \begin{equation}
    \label{eq:outputxell}
    s_n=\sum_{k=1}^{N}g(s_k+k-n),\quad \tn{for $n=1,2,\cdots,N$.}
    \end{equation}
    
    \begin{figure}[!htb]                
        \centering
        \includegraphics[width=\textwidth]{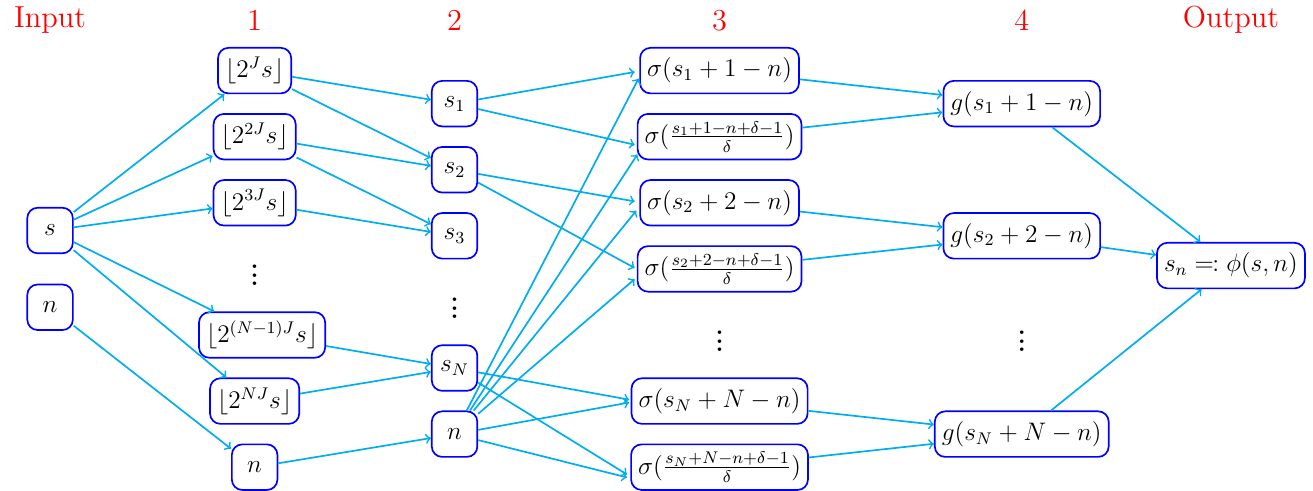}
        \caption{An illustration of the desired network architecture implementing $\phi$ based on Equation \eqref{eq:iterationFormula} and \eqref{eq:outputxell}. We omit some ReLU ($\sigma$) activation functions when inputs are obviously non-negative. \sj{All parameters in this network {are} essentially determined by Equation \eqref{eq:iterationFormula} and \eqref{eq:outputxell}, which are valid no matter what $\theta_1,\cdots,\theta_{NJ}\in \{0,1\}$ are. Thus, the desired function $\phi$ implemented by this network is independent of $\theta_1,\cdots,\theta_{NJ}\in \{0,1\}$.}}
        \label{fig:patchBitsExtration}
    \end{figure}

    As shown in Figure \ref{fig:patchBitsExtration}, the desired function $\phi$ can be  computed by a Floor-ReLU network   with  width $2N$ and  depth $4$. Moreover, it holds that
\begin{equation*}
\phi(s,\,n)=s_n,\quad \tn{for $n=1,2,\cdots,N$.}
\end{equation*}
 So we finish the proof.
\end{proof}

The next lemma constructs a Floor-ReLU network that can extract any bit from a binary representation according to a specific index. 
\begin{lemma}
    \label{lem:bitsExtraction}
    Given any $N,L\in \N^+$,
    there exists a function $\phi:\R^2\to \R$ implemented by a Floor-ReLU network  with  width $2N+2$ and depth $7L-3$ such that, for any $\theta_m\in \{0,1\}$, $m=1,2,\cdots,N^L$, we have
    \begin{equation*}
    \phi(\bin 0.\theta_1 \theta_2\cdots \theta_{N^L},\,m)=\theta_m,\quad \tn{for $m=1,2,\cdots,N^L$.}
    \end{equation*}
\end{lemma}

\begin{proof}
    The proof is based on repeated applications of Lemma \ref{lem:bitsExtractionBasic}. {Specifically, we inductively construct a sequence of functions $\phi_1,\phi_2,\cdots,\phi_L$ implemented by Floor-ReLU networks to satisfy the following two conditions for each $\ell \in \{1,2,\cdots,L\}$.}
    \begin{enumerate}[(i)]
        \item \label{cond1} $\phi_\ell:\R^2\to \R$ can be implemented by a Floor-ReLU network  with width $2N+2$ and depth $7\ell-3$.
        \item \label{cond2}For any $ \theta_{m}\in\{0,1\}$, $m=1,2,\cdots,N^\ell$, we have
        \begin{equation*}
        \phi_\ell(\bin 0.\theta_1\theta_2\cdots \theta_{N^\ell},\,m)=\bin 0.\theta_m,\quad \tn{for $m=1,2,\cdots,N^\ell$.}
        \end{equation*}
    \end{enumerate}

    Firstly, consider the case $\ell=1$. By Lemma \ref{lem:bitsExtractionBasic} (set $J=1$ therein), there exists a function  $\phi_1$ implemented by a Floor-ReLU network  with width \sj{$2N\le 2N+2$} and depth $4=7-3$ such that, for any $\theta_m\in\{0,1\}$, $m=1,2,\cdots,N$, we have
    \begin{equation*}
     \phi_1(\bin 0.\theta_1\theta_2\cdots\theta_N,\,m)=\bin 0.\theta_{m},\quad \tn{for $m=1,2,\cdots,N$.}
    \end{equation*}
    It follows that Condition \eqref{cond1} and \eqref{cond2} hold for $\ell=1$.
    
    Next, assume Condition \eqref{cond1} and \eqref{cond2} hold for $\ell=k$. We would like to construct $\phi_{k+1}$ to make Condition \eqref{cond1} and \eqref{cond2} true for $\ell=k+1$. By Lemma \ref{lem:bitsExtractionBasic} (set $J=N^k$ therein), there exists a function $\psi$ implemented by a Floor-ReLU network  with width \sj{$2N$} and  depth $4$ such that,
    for any $\theta_m\in \{0,1\}$, $m=1,2,\cdots,N^{k\sj{+1}}$, we have
    \begin{equation}
    \label{eq:induction1}
    \psi(\bin 0.\theta_1\cdots \theta_{N^{k+1}},\,n)=\bin 0.\theta_{(n-1)N^k+1}\cdots \theta_{(n-1)N^k+N^k},\quad \tn{for $n=1,2,\cdots,N$.}
    \end{equation}
    By the hypothesis of induction, we have 
    \begin{itemize}
        \item $\phi_k:\R^2\to \R$ can be implemented by a Floor-ReLU network  with width $2N+2$ and depth $7k -3$.
        
        \item For any $ \theta_j\in\{0,1\}
        $, $j=1,2,\cdots,N^k$, we have
        \begin{equation}
        \label{eq:induction2}
         \phi_k(\bin 0.\theta_1\theta_2\cdots \theta_{N^k},\,j)=\bin 0.\theta_j,\quad \tn{for $j=1,2,\cdots,N^k$.}
        \end{equation}
    \end{itemize}

    Given any $m\in \{1,2,\cdots,N^{k+1}\}$, there exist $n\in \{1,2,\cdots,N\}$ and $j\in \{1,2,\cdots,N^k\}$ such that $m=(n-1)N^k+j$, and such $n,j$ can be obtained by 
    \begin{equation}
    \label{eq:induction3}
    n=\lfloor (m-1)/N^k\rfloor+1 \quad \tn{and} \quad j=m-(n-1)N^k.
    \end{equation}
    Then the desired architecture of the Floor-ReLU network implementing $\phi_{k+1}$ is shown in Figure \ref{fig:bitsInduction}.
    \begin{figure}[!htb]        
        \centering
        \includegraphics[width=0.97\textwidth]{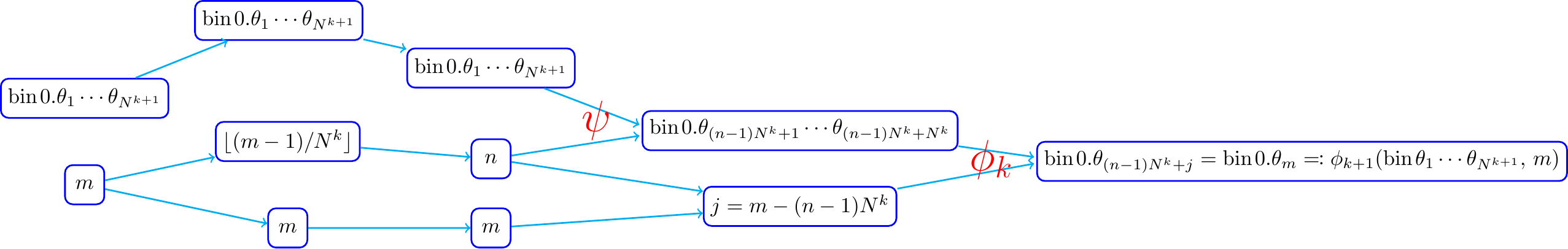}
        \caption{An illustration of the desired network architecture implementing $\phi_{k+1}$ based on \eqref{eq:induction1}, \eqref{eq:induction2}, and \eqref{eq:induction3}. We omit ReLU ($\sigma$) for neurons with non-negative inputs. }
        \label{fig:bitsInduction}
    \end{figure}
    
    Note that $\psi$ can be computed by a Floor-ReLU network of width $2N$ and depth $4$.
    By Figure \ref{fig:bitsInduction}, we have
    \begin{itemize}
        \item $\phi_{k+1}:\R^2\to \R$ can be implemented by a Floor-ReLU network with  width $2N+2$ and  depth $2+4+1+(7k-3)=7(k+1)-3$, which implies Condition \eqref{cond1}  for $\ell=k+1$.
        \item For any $\theta_m\in \{0,1\}$, $m=1,2,\cdots,N^{k+1}$, we have
        \begin{equation*}
        \phi_{k+1}(\bin 0.\theta_1\theta_2\cdots \theta_{N^{k+1}},\,m)=\bin 0.\theta_m,\quad \tn{for $m=1,2,\cdots,N^{k+1}$.}
        \end{equation*} 
        That is, Condition  \eqref{cond2} holds for $\ell=k+1$.
    \end{itemize}
    So we finish the process of induction.
    
    By the principle of induction, there exists a function $\phi_L:\R^2\to \R$  such that
        \begin{itemize}
        \item  $\phi_L$ can be implemented by a Floor-ReLU network  with width $2N+2$ and  depth $7L-3$.
        \item For any $ \theta_m\in\{0,1\}$, $m=1,2,\cdots,N^L$, we have
        \begin{equation*}
        \phi_L(\bin 0.\theta_1\theta_2\cdots \theta_{\sj{N}^L},\,m)=\bin 0.\theta_m,\quad \tn{for $m=1,2,\cdots,N^L$.}
        \end{equation*}
    \end{itemize}
    Finally, define $\phi\coloneqq 2 \phi_L$. Then $\phi$ can also be implemented by  a Floor-ReLU network  with width $2N+2$ and  depth $7L-3$. Moreover, for any $ \theta_m\in\{0,1\}$, $m=1,2,\cdots,N^L$, we have
     \begin{equation*}
     \phi(\bin 0.\theta_1\theta_2\cdots \theta_{N^L},\,m)=2\cdot\phi_L(\bin 0.\theta_1\theta_2\cdots \theta_{N^L},\,m)=2\cdot\bin 0.\theta_m=\theta_m,
     \end{equation*}
     \tn{for $m=1,2,\cdots,N^L$.} So we finish the proof.
\end{proof}

With Lemma \ref{lem:bitsExtraction} in hand, we are ready to prove Proposition \ref{prop:bitsExtraction}.
\begin{proof}[Proof of Proposition \ref{prop:bitsExtraction}]
    By Lemma \ref{lem:bitsExtraction}, there exists a function $\tildephi:\R^2\to \R$ computed by a Floor-ReLU network with  a fixed architecture with width $2N+2$ and  depth $7L-3$ such that, for any $z_m\in \{0,1\}$, $m=1,2,\cdots,N^L$, we have
    \begin{equation*}
    \tildephi(\bin 0.z_1z_2\cdots z_{N^L},\,m)=z_m,\quad \tn{for $m=1,2,\cdots,N^L$.}
    \end{equation*}
    Based on  $\theta_m\in \{0,1\}$ for $m=1,2,\cdots,N^L$ given in Proposition \ref{prop:bitsExtraction}, we define the final function $\phi$ as 
    \begin{equation*}
    \phi(x)\coloneqq \tildephi \big(\sigma( x\cdot 0+\bin 0.\theta_1\theta_2\cdots \theta_{N^L}),\sigma(x)\big),\quad \tn{where $\sigma(x)=\max\{0,x\}$.}
    \end{equation*}
     Clearly, $\phi$ can be implemented by a Floor-ReLU network  with width $2N+2$ and  depth $(7L-3)+1=7L-2$. Moreover, we have, for any $m\in \{1,2,\cdots,N^L\}$,
    \begin{equation*}
    \phi(m)\coloneqq \widetilde{\phi}\big(\sigma(m\cdot 0+\bin 0.\theta_1\theta_2\cdots \theta_{N^L}),\sigma(m)\big)=\widetilde{\phi}(\bin 0.\theta_1\theta_2\cdots \theta_{N^L},m)=\theta_m.
    \end{equation*}
    So we finish the proof.
\end{proof}

We {finally} point out that only the properties of Floor on $[0,\infty)$ are used in our proof. Thus, the Floor can be replaced by the truncation function that can be easily computed by truncating the decimal part.

\section{Conclusion}
\label{sec:conclusion}

This paper has introduced a theoretical framework to show that deep network approximation can achieve \sj{root} exponential convergence and avoid the curse of dimensionality for approximating functions as general as (H\"older) continuous functions. 
Given a 
Lipschitz continuous function $f$ on $[0,1]^d$, it was shown by construction that Floor-ReLU networks with width $\max\{d,\, 5N+13\}$ and  depth $64dL+3$ {can achieve a
uniform approximation error bounded by} $3\lambda\sqrt{d}\,N^{-\sqrt{L}}$, where $\lambda$ is the Lipschitz constant of $f$. More generally for an arbitrary continuous function $f$ on $[0,1]^d$ with a modulus of continuity $\omega_f(\cdot)$, {the approximation error is bounded by} $\omega_f(\sqrt{d}\,N^{-\sqrt{L}})+2\omega_f(\sqrt{d}){N^{-\sqrt{L}}}$. \sj{The results in this paper provide a theoretical lower bound of the power of deep network approximation. Whether or not this bound is achievable in actual computation relies on advanced algorithm design as a separate line of research. }

\vspace{0.5cm}

{\bf Acknowledgments.} Z.~Shen is supported by Tan Chin Tuan Centennial Professorship.  H.~Yang was partially supported by the US National Science Foundation under award DMS-1945029.

\bibliographystyle{apalike}
\bibliography{references}
\end{document}